%% file: main.tex
\documentclass{article}

 \usepackage[preprint]{neurips_2026}

\usepackage[utf8]{inputenc} % allow utf-8 input
\usepackage[T1]{fontenc}    % use 8-bit T1 fonts
\usepackage{hyperref}       % hyperlinks
\usepackage{url}            % simple URL typesetting
\usepackage{booktabs}       % professional-quality tables
\usepackage{amsfonts}       % blackboard math symbols
\usepackage{nicefrac}       % compact symbols for 1/2, etc.
\usepackage{microtype}      % microtypography
\usepackage[table,dvipsnames]{xcolor}         % colors

\usepackage[ruled,linesnumbered]{algorithm2e}
\usepackage{amsmath, amssymb, mathtools, amsthm}
\usepackage{todonotes}
\usepackage{xspace}
\usepackage{enumitem}
\usepackage{multirow, multicol}
\usepackage{wrapfig}
\usepackage{graphicx}
\usepackage{bbm}
\usepackage{pifont}
\usepackage{stackengine}

\theoremstyle{definition}
\newtheorem{theorem}{Theorem}
\newtheorem{lemma}{Lemma}

\newtheorem{assumption}{Assumption}
\newtheorem{remark}{Remark}
\newcommand{\ours}{\texttt{FedSPM}\xspace}
\setcitestyle{numbers,square,comma,sort&compress}

\newcommand{\ie}{{\em i.e.,~}}
\title{\ours: Routing-Enabled Federated Learning under Dual Heterogeneity via Semiparametric Mixture}

\author{%
  Zijian Wang$^{1}$ \quad Pengfei Li$^{2}$ \quad Guangyu Yang$^{1}$  \quad Qiong Zhang$^{1}$\thanks{Correspondence to: Qiong Zhang (\texttt{qiong.zhang@ruc.edu.cn}) and Guangyu Yang (\texttt{yguangyu@ruc.edu.cn})} \\[0.5em]
  $^{1}$Institute of Statistics and Big Data, Renmin University of China\\
  $^{2}$Department of Statistics and Actuarial Science, University of Waterloo
}

\begin{document}
\maketitle

\begin{abstract}
Routing-prediction federated learning has emerged as a new paradigm that reframes inter-client heterogeneity as a resource for system-level intelligence: at inference time, the server routes each external query to the best-matched client for prediction.
Existing approaches, however, typically treat each client as internally homogeneous, overlooking latent subpopulations within local data.
For example, patients with the same diagnosis at one hospital may exhibit morphologically distinct disease subtypes.
The coexistence of inter-client and intra-client heterogeneity, which we call \emph{dual heterogeneity}, can impair both routing and prediction.
To address this challenge, we propose \ours, a routing-enabled semiparametric mixture framework that represents each client using client-specific latent components.
Each component combines a predictive distribution for classification with a feature distribution for routing.
To flexibly model feature distributions while effectively sharing information across clients, \ours models their density ratios relative to a common nonparametric measure estimated via empirical likelihood.
We develop a federated expectation-maximization algorithm that optimizes a tractable surrogate and prove convergence of the exact profiled objective at the standard $\mathcal{O}(1/\sqrt{T})$ rate when the surrogate errors are properly controlled.
Experiments on controlled benchmarks and real-world medical data demonstrate consistent improvements in routing and prediction under dual heterogeneity.
Code is available \href{https://github.com/zijianwang0510/FedSPM}{here}.
\end{abstract}

\input{sections/introduction}
\input{sections/method}
\input{sections/experiment}
\input{sections/conclusion}

\section*{Acknowledgment}
Zijian Wang \& Qiong Zhang are supported by the National Key R\&D Program of China Grant 2024YFA1015800 and the National Natural Science Foundation of China Grant 12301391.
Pengfei Li is supported in part by the Natural Sciences and Engineering Research Council of Canada (RGPIN-2026-04406) and the Faculty of Mathematics Research Chair at the University of Waterloo.

%%%%%%%%%%%%%%%%%%%%%%%%%%%%%%%%%%%%%%%%%%%%%%%%%%%%%%%%%%%%
\bibliographystyle{abbrvnat}
\bibliography{reference}

%%%%%%%%%%%%%%%%%%%%%%%%%%%%%%%%%%%%%%%%%%%%%%%%%%%%%%%%%%%%
\appendix
\input{sections/appendix}

%%%%%%%%%%%%%%%%%%%%%%%%%%%%%%%%%%%%%%%%%%%%%%%%%%%%%%%%%%%%
\end{document}

%% file: sections/introduction.tex
\section{Introduction}
Federated learning (FL)~\cite{mcmahan2017fedavg} enables multiple clients to collaboratively train a model without centralizing their raw data.
In a typical FL system, a server distributes a shared model to participating clients, each client updates the model using its local data, and the server aggregates the resulting locally updated models into an improved global model.
This distributed training paradigm allows knowledge to be shared across clients such as hospitals~\cite{xu2021application} and mobile devices~\cite{li2020application}, making FL particularly promising when data are sensitive, geographically dispersed, or impractical to centralize~\cite{li2020survey,kairouz2021survey,yang2019federated}. 

Traditionally, \emph{inter-client heterogeneity}, where data distributions vary across clients, is viewed as an obstacle in FL.
Since clients optimize different local objectives, their gradients may drift in conflicting directions~\cite{karimireddy2020scaffold}, thereby slowing convergence~\cite{li2020theory} and degrading global model performance~\cite{zhao2018federated}.
In contrast, the recent \emph{routing-prediction} FL paradigm~\cite{wang2026feddrm} reframes inter-client heterogeneity as a useful signal of client specialization.
Like personalized FL~\cite{li2021ditto,ghosh2020ifca,arivazhagan2019fedper}, it learns specialized models for individual client domains. 
Beyond personalization, it further estimates how well an external query matches each client's data distribution.
At inference time, the server uses these distributional match scores to route the query to the most suitable client, whose specialized model then makes the final prediction, thereby turning client-specific expertise into system-level intelligence.

However, such a routing-prediction FL paradigm typically assumes that the data within each client are drawn from a homogeneous distribution, thereby overlooking \emph{intra-client heterogeneity}. 
This assumption is often violated in practice, as local data may arise from a mixture of latent components~\cite{marfoq2021fedem,wu2023fedgmm}. 
For example, within one hospital, dermoscopic images may involve different lesion types, anatomical sites, and patient age groups: cases with the same diagnosis may exhibit markedly different visual patterns~\cite{zalaudek2006age,changchien2007age}, whereas visually similar cases may correspond to different diagnoses~\cite{braga2008melanoma,carrera2017dermoscopic}.
Since these factors and their interactions are rarely fully observed or annotated, the resulting subgroup memberships are latent, making it infeasible to fit a separate model to each predefined group.
Ignoring such latent structure forces a homogeneous local model to fit a mixture of heterogeneous feature and predictive distributions, thereby degrading both routing and prediction accuracy. 
Taken together, intra-client and inter-client heterogeneity constitute what we call \emph{\underline{dual heterogeneity}}.
This raises our key question: \emph{how can routing and prediction be jointly improved under dual heterogeneity?}

No existing heterogeneous FL approach fully addresses this problem.
Methods centered on a global model mitigate client drift through regularization~\cite{li2020fedprox,acar2021feddyn}, aggregation reweighting~\cite{wang2020fednova,li2023fedlaw}, or refined optimization~\cite{hsu2019fedavgm,reddi2021fedopt}, but their shared predictor still struggles to adapt to client-specific distributions.
Conventional personalized methods, including local fine-tuning~\cite{wang2019ft}, regularization~\cite{dinh2020pfedme,li2021ditto}, client clustering~\cite{ghosh2020ifca,sattler2021clusterfl}, and representation learning~\cite{arivazhagan2019fedper,oh2022fedbabu}, improve prediction within each client's local domain, but provide no mechanism for server-side routing of external queries.
Routing-based 
personalization~\cite{wang2026feddrm} enables server-side routing, yet treats each client as internally homogeneous and therefore overlooks intra-client heterogeneity.
Mixture-model-based personalization captures intra-client heterogeneity~\cite{marfoq2021fedem,wu2023fedgmm}, but its assumption of shared component distributions across clients limits flexibility in modeling client-specific latent structures.

\begin{figure}[!ht]
    \centering
    \vspace{-0.5\baselineskip}
    \includegraphics[width=0.9\linewidth]{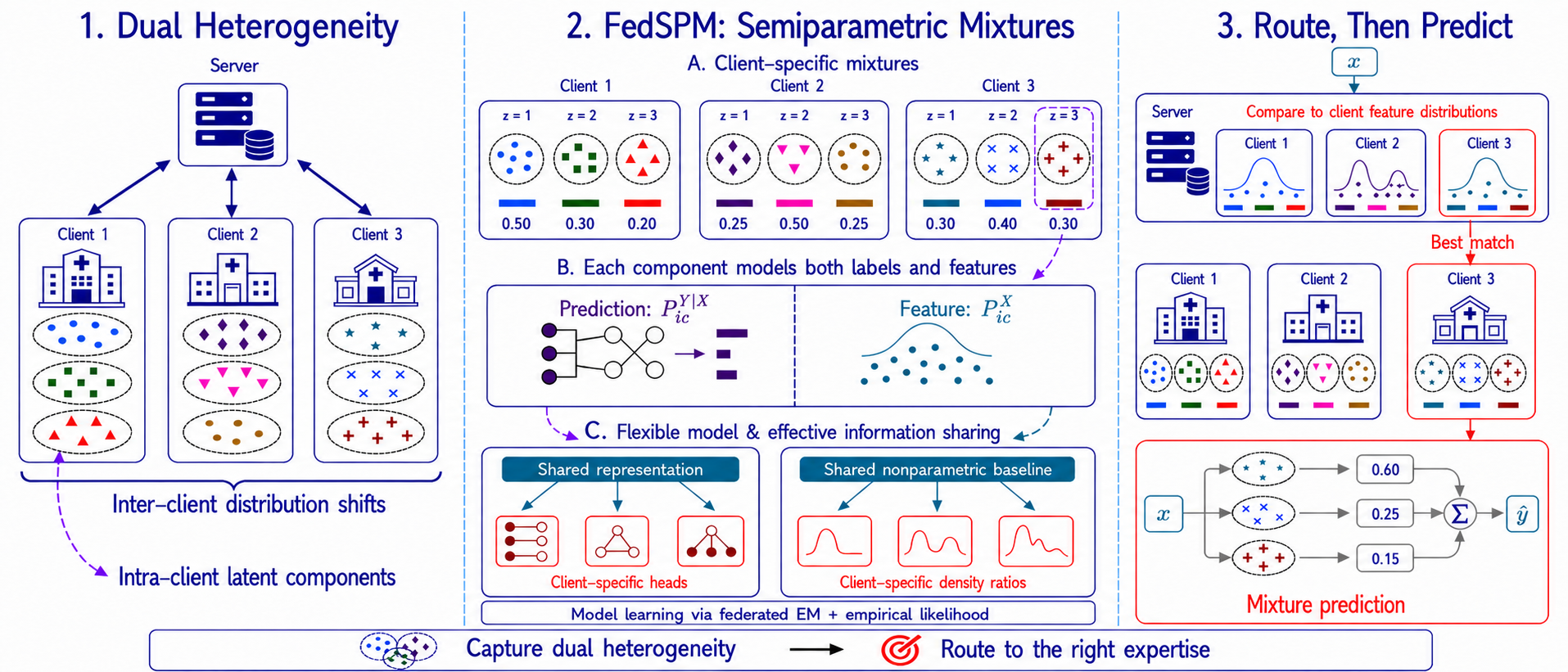}
    \caption{\textbf{Overview of the \ours framework.}}
    \vspace{-0.5\baselineskip}
    \label{fig:FedSPM}
\end{figure}

To address this challenge, we propose \ours (Fig.~\ref{fig:FedSPM}), a routing-enabled FL framework that represents each client as a mixture of latent components rather than a single homogeneous population.
This mixture discovers hidden variation within each client, while allowing its components to differ across clients captures inter-client distribution shifts.
Each component contains two complementary parts: a predictive distribution that relates features to labels and a feature distribution that characterizes the inputs covered by the component.
To balance model flexibility with effective information sharing, the predictive distributions combine shared representations with client-specific prediction heads, while the feature distributions use a density ratio model (DRM)~\cite{anderson1979drm} relative to a shared nonparametric baseline.
We learn the model using empirical likelihood (EL)~\cite{owen1990empirical} and a federated expectation-maximization (EM) algorithm.
At inference time, the server routes an external query to the most suitable client using the learned feature distributions, and the selected client combines its component-wise predictions to produce the final output.
Our main contributions are summarized as follows:
\begin{itemize}[leftmargin=*, topsep=2pt, itemsep=2pt, parsep=0pt]
\item We formulate routing-prediction FL under dual heterogeneity and propose \ours, a semiparametric mixture framework that flexibly captures latent variation within clients and distribution shifts across clients while enabling effective information sharing.

\item We develop a federated EM algorithm based on EL and establish convergence guarantees for the resulting nonconvex optimization under local stochastic gradient descent (SGD) with momentum.

\item We evaluate \ours on controlled benchmarks and a real-world medical dataset, demonstrating consistent improvements in both routing and prediction over competitive FL baselines.
\end{itemize}

%% file: sections/method.tex
\section{Method}

\textbf{Problem Formulation.}
Consider an FL system with $m$ clients for $K$-class classification.
Each client $i \in [m]\coloneqq\{1,\dots, m\}$ has a local dataset $\mathcal{D}_i\coloneqq\{(x_{ij},y_{ij})\}_{j\in [n_i]}$ with samples drawn independently from a client-specific distribution $P_i^{X,Y}$.
Let $n\coloneqq \sum_{i=1}^m n_i$ be the total sample size, $\rho_i\coloneqq n_i/n$ be the sample fraction of client $i$, and $\mathcal{D}\coloneqq\bigcup_{i=1}^m\mathcal{D}_i$ be the pooled dataset.

\subsection{Semiparametric Mixture Model}
To capture \emph{intra-client heterogeneity}, we introduce a latent component variable $Z\in[C]$ and model each observed distribution $P_i^{X,Y}$ through an augmented distribution $P_i^{X,Y,Z}$. 
Specifically, we assume
\begin{equation}\label{eq:mixture}
(x_{ij},y_{ij})\mid(z_{ij}=c)\sim P_{ic}^{X,Y},~z_{ij}\sim P_i^Z \text{ with } P_i^Z(\{c\})=\pi_{ic},~\forall c\in[C],~j\in[n_i].
\end{equation}
Here, $\pi_i\coloneqq(\pi_{i1},\dots,\pi_{iC})$ with $\pi_{ic}\ge0$ and $\sum_c\pi_{ic}=1$.
Marginalizing over $Z$ gives $P_i^{X,Y}=\sum_c\pi_{ic}P_{ic}^{X,Y}$.
Notably, we allow both the mixing weights $\pi_i$ and the component distributions $P_{ic}^{X,Y}$ to vary across clients, providing a flexible model of intra-client heterogeneity.

While allowing fully client-specific $P_{ic}^{X,Y}$ is expressive, it prevents effective information sharing across clients. 
We factor $P_{ic}^{X,Y}$ into its predictive distribution $P_{ic}^{Y|X}$ and feature distribution $P_{ic}^X$, and impose structure on both factors to share information while modeling \emph{inter-client heterogeneity}:

\textbf{Concept shift.}
We model the component predictive distribution of client $i$ as
\begin{equation}\label{eq:concept_shift}
P_{ic}^{Y|X}(\{k\}\mid x)\propto\exp(\alpha_{ikc}+\beta_{ikc}^\top g_{\theta_c}(x)), 
\end{equation}
where $g_{\theta_c}(x)$ is a \emph{shared} embedding for component $c$. 
The client-specific parameters $(\alpha_{ikc}, \beta_{ikc})$ allow the marginal predictive distribution $P_i^{Y|X}$ to vary across clients, hence capturing concept shift.

\textbf{Covariate shift.}
We model the component feature distribution of client $i$ via a DRM with respect to a shared baseline distribution $G$:
\begin{equation}\label{eq:covariate_shift}
dP_{ic}^{X}/dG(x) = \exp\big(\gamma_{ic} + \xi_{ic}^\top h_{\nu_c}(x)\big),
\end{equation}
where $dP_{ic}^{X}/dG$ is the Radon--Nikodym derivative of $P_{ic}^X$ relative to $G$, and $h_{\nu_c}$ is a \emph{shared} basis for component $c$. 
The client-specific tilting parameters $(\gamma_{ic}, \xi_{ic})$ induce differences in the marginal feature distribution $P_i^X$ across clients, capturing covariate shift while preserving a shared structure.

\textbf{Label shift.}
Finally, differences in the mixing weights $\pi_i$ and the component distribution $P_{ic}^{X,Y}$ jointly induce variation in the marginal label distribution $P_i^Y$ across clients, capturing label shift.

The baseline distribution $G$ in \eqref{eq:covariate_shift} remains unspecified.
Restricting $G$ to a parametric family imposes an unjustified distributional assumption, which limits model flexibility and increases the risk of misspecification.
Instead, we estimate $G$ via EL, a nonparametric likelihood framework that assigns unknown probability masses to the observed samples, subject to the DRM constraints (see App.~\ref{app:related_work} for related work on DRM and EL).
Specifically, $G$ is represented as a discrete distribution supported on the pooled observations, taking the form $G=\sum_{i,j}r_{ij}\delta_{x_{ij}}$ with unknown $r_{ij} \ge 0$.
Crucially, \emph{all samples across all clients jointly determine $G$}, serving as a key mechanism for information sharing.
To ensure that $G$ and each $P_{ic}^X$ define valid probability distributions, the masses $\{r_{ij}\}$ must satisfy $\int G(dx)=1$ and $\int \exp\left(\gamma_{i'c}+\xi_{i'c}^\top h_{\nu_c}(x)\right) G(dx)=1$ for all $i',c$, \ie
\begin{equation}\label{eq:constraints}
\sum_{i,j} r_{ij} = 1, \quad
\sum_{i,j} \exp\left(\gamma_{i'c} + \xi_{i'c}^\top h_{\nu_c}(x_{ij})\right) r_{ij} = 1,
\quad \forall i' \in [m], c \in [C].
\end{equation}
In summary, the predictive model \eqref{eq:concept_shift}, the DRM-EL feature model \eqref{eq:covariate_shift}, and the mixing weights $\pi_i$ together define our \emph{semiparametric mixture model}.

\begin{remark}[Identifiability and interpretability]
Our model is intended as a flexible approximation to the client-specific joint distribution $P_i^{X,Y}$, rather than as a tool for recovering identifiable or interpretable latent structures.
Its non-identifiability arises from both the neural network implementations of the representation maps $g_{\theta_c}$ and $h_{\nu_c}$, and the mixture structure itself. 
For the latter, if two components on client $i$ collapse, \ie $P_{ic_1}^{X,Y}=P_{ic_2}^{X,Y}$ for some $c_1\neq c_2$, then redistributing mass between $\pi_{ic_1}$ and $\pi_{ic_2}$ leaves $P_{i}^{X,Y}=\sum_c\pi_{ic}P_{ic}^{X,Y}$ unchanged.
Thus, the learned components should be viewed as auxiliary constructs for routing and prediction, not as recovered true subgroups.
\end{remark}

\begin{remark}[Comparison to existing work]
\label{remark:diff_prior_work}
Prior work such as~\cite{marfoq2021fedem, wu2023fedgmm} assume $P_{ic}^{X,Y}=P_{jc}^{X,Y}$ for all $i, j, c$, \ie identical component distributions across clients. 
This assumption is often unrealistic. For instance, even data from clinically similar patient groups can still differ substantially across hospitals due to site-specific acquisition protocols, imaging devices, and preprocessing pipelines.
\end{remark}

%%%%%%%%%%%%%%%%%%%%%%%%%%%%%%%%%%%%%%%%%%%%%%%%%%%%%%%%%%%%%%%%%%%%%

\subsection{Profile Log-EL and EM Algorithm}
We develop a practical learning procedure for the proposed semiparametric mixture model.
The main difficulty is that both the latent component assignments and the nonparametric baseline distribution $G$ are unknown.
We first profile out $G$ through EL, and then exploit the latent mixture structure to derive an EM algorithm that avoids direct optimization of the resulting intractable objective.

\textbf{Profile log-EL.}
Let $\zeta\coloneqq(\alpha,\beta,\theta,\gamma,\xi,\nu,\pi)$ be the parameters of interest. 
The log-EL based on $\mathcal{D}$ is:
\begin{equation*}
    \ell(\zeta, G) = \sum_{i,j} \log r_{ij} + \sum_{i,j} \log \biggl[ \sum_{c} \pi_{ic} \exp \left( \gamma_{ic} + \xi_{ic}^\top h_{\nu_c}(x_{ij}) \right) P_{ic}^{Y|X}(\{y_{ij}\}\mid x_{ij}) \biggr].
\end{equation*}
See App.~\ref{app:log-EL} for the derivation.
Since $G$ is a nuisance parameter, we work with the profile log-EL $p\ell(\zeta)\coloneqq\sup_G \ell(\zeta, G)$, where the supremum is taken under the constraints in \eqref{eq:constraints}.
By the Lagrange multiplier method, the optimal baseline weights are:
\begin{equation}\label{eq:optimal_G}
r_{ij}^*(\zeta)=n^{-1}\biggl\{1+\sum_{i',c}\lambda_{i'c}\left[\exp\left(\gamma_{i'c}+\xi_{i'c}^\top h_{\nu_c}(x_{ij})\right)-1\right] \biggr\}^{-1},
\end{equation}
where $\{\lambda_{ic}\}$ are the solutions to
\begin{equation}
\label{eq:Lagrange_multipliers_system}
\sum_{i,j}\frac{\exp\left(\gamma_{i'c}+\xi_{i'c}^\top h_{\nu_c}(x_{ij})\right)-1}{1+\sum_{i'',c'}\lambda_{i''c'}\left[\exp\left(\gamma_{i''c'}+\xi_{i''c'}^\top h_{\nu_{c'}}(x_{ij})\right)-1\right]}=0,\quad\forall i'\in [m], c\in [C].
\end{equation}
See App.~\ref{app:optimal_G} for the derivation. 
Substituting \eqref{eq:optimal_G} back into $\ell(\zeta,G)$ yields $p\ell(\zeta)$.
However, directly maximizing $p\ell(\zeta)$ remains computationally impractical because each evaluation requires solving the nonlinear system in \eqref{eq:Lagrange_multipliers_system} for the Lagrange multipliers $\{\lambda_{ic}\}$, making end-to-end optimization with standard automatic-differentiation tools such as \textit{PyTorch} infeasible.
 
\textbf{EM algorithm.}
To address this challenge, we treat the component assignments $\{z_{ij}\}$ as missing data and employ the EM algorithm to maximize $p\ell(\zeta)$.
Let $z_{ijc}\coloneqq\mathbbm{1}(z_{ij}=c)$ denote the indicator that the $j$-th sample on the $i$-th client belongs to the $c$-th subpopulation.
If $\{z_{ij}\}$ were observed, the complete-data profile log-EL would be
\begin{equation*}
p\ell^{\text{c}}(\zeta)=\sum_{i,j}\log r_{ij}^*(\zeta)+\sum_{i,j,c}z_{ijc}\biggl[\log P_{ic}^{Y|X}(\{y_{ij}\}\mid x_{ij})+\gamma_{ic}+\xi_{ic}^\top h_{\nu_c}(x_{ij})+\log \pi_{ic} \biggr].
\end{equation*}
See App.~\ref{app:complete-data_profile_log-EL} for derivation.
Since $\{z_{ij}\}$ are unobserved, EM alternates between estimating their posterior distributions under the current parameters and updating the parameters based on these estimates.
Specifically, the two steps are:

\textbf{\textit{E-step.}}
Given $\zeta^{(t)}$, we compute the posterior responsibilities for each $z_{ij}$ as
\begin{equation}
\label{eq:posterior}
w_{ijc}^{(t)} \coloneqq \mathbb{E}\left[z_{ijc}\middle| \mathcal{D};\zeta^{(t)}\right] \propto P_{ic}^{Y|X}(\{y_{ij}\}\mid x_{ij};\zeta^{(t)})\exp\left(\gamma_{ic}^{(t)}+(\xi_{ic}^{(t)})^\top h_{\nu_c^{(t)}}(x_{ij})\right)\pi_{ic}^{(t)}.
\end{equation}
The resulting profile $Q$-function is the conditional expectation of the complete-data profile log-EL:
\begin{equation}\label{eq:Q_decomposition}
Q^{(t)}(\zeta)\coloneqq\mathbb{E}\left[p\ell^{\text{c}}(\zeta)\middle|\mathcal{D};\zeta^{(t)}\right] = Q^{(t)}_1(\alpha, \beta, \theta) + Q^{(t)}_2(\gamma, \xi, \nu) + Q^{(t)}_3(\pi), 
\end{equation}
where
\begin{equation*}
Q_1^{(t)}(\alpha,\beta,\theta) \coloneqq \sum_{i,j,c}w_{ijc}^{(t)}\log P_{ic}^{Y|X}(\{y_{ij}\}\mid x_{ij}), \quad Q_3^{(t)}(\pi) \coloneqq \sum_{i,j,c}w_{ijc}^{(t)}\log\pi_{ic}, \text{ and }
\end{equation*}
\begin{equation*}
\resizebox{\linewidth}{!}{
    $\displaystyle Q_2^{(t)}(\gamma, \xi, \nu) \coloneqq \sum_{i,j,c}w_{ijc}^{(t)} \left[\gamma_{ic} + \xi_{ic}^\top h_{\nu_c}(x_{ij})\right] - \sum_{i,j}\log \biggl(1 + \sum_{i',c} \lambda_{i'c} \left[\exp\left(\gamma_{i'c} + \xi_{i'c}^\top h_{\nu_c}(x_{ij})\right) - 1\right]\biggr).$
    }
\end{equation*}

\textbf{\textit{M-step.}}
We update $\zeta^{(t+1)}=\arg\max Q^{(t)}(\zeta)$.
Since $Q_1^{(t)}$, $Q_2^{(t)}$, and $Q_3^{(t)}$ are separable, the three parameter blocks can be optimized independently.
For convenience, we refer to $(\alpha,\beta,\theta)$ as the \emph{supervised block}, and to $(\gamma,\xi,\nu)$ as the \emph{DRM block}.
For the mixing weights, maximizing $Q^{(t)}_3$ under $\sum_c \pi_{ic} = 1$ yields
\begin{equation}
\label{eq:update_pi}
\pi_{ic}^{(t+1)}=\frac{1}{n_i}\sum_{j=1}^{n_i}w_{ijc}^{(t)}.
\end{equation}
See App.~\ref{app:EM_algorithm_details} for derivation details of the EM algorithm.
For the supervised block, maximizing $Q_1^{(t)}$ is equivalent to minimizing a weighted cross-entropy loss, which can be directly solved via standard gradient-based methods:
$\left(\alpha^{(t+1)},\beta^{(t+1)},\theta^{(t+1)}\right)=\arg\max Q^{(t)}_1(\alpha,\beta,\theta)$.
For the DRM block, directly maximizing $Q_2^{(t)}$ faces the same computational bottleneck as $p\ell(\zeta)$: each evaluation still requires solving the nonlinear system for the Lagrange multipliers $\{\lambda_{ic}\}$. 
In contrast, the EM algorithm offers a pivotal resolution to this bottleneck: at any critical point of $Q^{(t)}_2$, each $\lambda_{ic}$ admits the analytic form $n^{-1}\sum_{j=1}^{n_i}w_{ijc}^{(t)}$ (see App.~\ref{app:analytical_form_of_lagrange_multipliers}). 
Substituting this expression back shows that every critical point of $Q^{(t)}_2$ is also that of the tractable surrogate
\begin{equation*}
\resizebox{\linewidth}{!}{
$\displaystyle\widetilde{Q}_2^{(t)}(\gamma,\xi,\nu)\coloneqq \sum_{i,j,c}w_{ijc}^{(t)}\left[\gamma_{ic}+\xi_{ic}^\top h_{\nu_c}(x_{ij})\right]-\sum_{i,j}\log\biggl(\sum_{i',j',c}w_{i'j'c}^{(t)}\exp\left(\gamma_{i'c}+\xi_{i'c}^\top h_{\nu_c}(x_{ij})\right)\biggr).$}
\end{equation*}
We optimize $\widetilde{Q}^{(t)}_2$ instead, giving the update $\left(\gamma^{(t+1)},\xi^{(t+1)},\nu^{(t+1)}\right)=\arg\max\widetilde{Q}_2^{(t)}(\gamma,\xi,\nu)$.

\subsection{Federated Training and Inference}
\label{sec:fl}

Algorithm~\ref{algo:fedspm} in App.~\ref{app:algorithm} presents \ours, which extends the centralized EM procedure to the client-server setting.
As in standard FL, \ours follows a local-update-then-aggregate paradigm.
In communication round $t$, the server broadcasts the latest shared parameters $(\theta^{(t)},\gamma^{(t)},\xi^{(t)},\nu^{(t)})$ to all clients.
Each client $i$ then performs a local E-step followed by a local M-step:

\textbf{\textit{Local E-step.}}
Client $i$ computes $\{w_{ijc}^{(t)}\}_{j\in[n_i],c\in[C]}$ via \eqref{eq:posterior} using only its own dataset $\mathcal{D}_i$.

\textbf{\textit{Local M-step.}}
After the local E-step, client $i$ updates the three parameter blocks separately. 
First, since the mixing weights $\pi_i$ are client-specific, client $i$ updates them locally via \eqref{eq:update_pi} without server aggregation.
For the supervised block, client $i$ updates $(\alpha_i^{(t+1)},\beta_i^{(t+1)},\theta_i^{(t+1)})$ by maximizing
\begin{equation*}
Q_{1i}^{(t)}(\alpha, \beta, \theta) \coloneqq \sum_{j,c} w_{ijc}^{(t)} \log P_{ic}^{Y|X}(\{y_{ij}\}\mid x_{ij}),
\end{equation*}
which depends only on $\mathcal{D}_i$ and $\{w_{ijc}^{(t)}\}_{j\in[n_i],c\in[C]}$ available on client $i$.
For the DRM block, client $i$ updates $(\gamma_i^{(t+1)}, \xi_i^{(t+1)}, \nu_i^{(t+1)})$ by maximizing 
\begin{equation*}
\widetilde{Q}_{2i}^{(t)}(\gamma,\xi,\nu)\coloneqq \sum_{j,c}w_{ijc}^{(t)}\left[\gamma_{ic}+\xi_{ic}^\top h_{\nu_c}(x_{ij})\right]-\sum_{j}\log\biggl(\sum_{i',c}\tau_{i'c}^{(t)}\exp\left(\gamma_{i'c}+\xi_{i'c}^\top h_{\nu_c}(x_{ij})\right)\biggr),
\end{equation*}
where $\tau_{i'c}^{(t)}\coloneqq\sum_{j'=1}^{n_{i'}}w_{i'j'c}^{(t)}$ denotes the total responsibility assigned to component $c$ on client $i'$.
Unlike $Q_{1i}^{(t)}$, client $i$ cannot evaluate $\widetilde{Q}_{2i}^{(t)}$ using only its local information, because $\widetilde{Q}_{2i}^{(t)}$ involves responsibilities from other clients.
However, this cross-client dependence enters only through the summary statistics $\{\sum_{j'=1}^{n_{i'}}w_{i'j'c}^{(t)}\}_{i'\in[m],c\in[C]}$.
Accordingly, after the local E-step, each client $i$ transmits $\{\tau_{ic}^{(t)}\}_{c\in[C]}$ to the server, and the server broadcasts the collection $\{\tau_{ic}^{(t)}\}_{i\in[m],c\in[C]}$ to all clients.
This additional communication consists of only $m\times C$ scalars, which is negligible relative to model transmission and does not reveal raw data or per-sample responsibilities. 
Armed with these summary statistics, client $i$ can maximize both $Q_{1i}^{(t)}$ and $\widetilde{Q}_{2i}^{(t)}$ via any gradient-based \texttt{LocalSolver}.

\textbf{Routing and prediction.}
At inference time, given an external query $x$, the server first routes it to the most suitable client using maximum a posteriori estimation, with the sample fraction $\rho_i$ as the prior and the client feature density relative to the baseline distribution $G$ as the likelihood:
\begin{equation}\label{eq:client_routing}
i^*\coloneqq \arg\max_{i\in [m]}\left\{\rho_i \cdot dP_i^{X}/dG(x)\right\} = \arg\max_{i\in [m]}\biggl\{\rho_i\sum_{c}\pi_{ic}\exp\big(\gamma_{ic} + \xi_{ic}^\top h_{\nu_c}(x)\big)\biggr\}.
\end{equation}
Let $P_i^{Z|X}$ denote the conditional distribution of the latent component given the query on client $i$.
The selected client $i^*$ then makes the final prediction:
\begin{equation}\label{eq:local_prediction}
\widehat{y}\coloneqq \arg\max_{k\in [K]} P_{i^*}^{Y|X}(\{k\}\mid x)= \arg\max_{k\in [K]}\biggl\{\sum_{c} P_{i^*c}^{Y|X}(\{k\}\mid x)P_{i^*}^{Z|X}(\{c\}\mid x)\biggr\}.
\end{equation}

\subsection{Convergence Analysis}
We analyze the convergence of \ours when \texttt{LocalSolver} is instantiated as local SGD~\cite{stich2018localsgd} with momentum. 
Since the mixing weights and the supervised block follow standard EM/GEM updates and inherit the convergence guarantees established by~\cite{marfoq2021fedem}, we focus exclusively on the DRM block $\phi\coloneqq(\gamma,\xi,\nu)$.
For convenience, we cast the optimization of $\phi$ as a standard minimization problem:
\begin{equation*}
    F(\zeta)\coloneqq -p\ell(\zeta),\quad
f^{(t)}(\phi)\coloneqq -Q_2^{(t)}(\phi),\quad
\widetilde f^{(t)}(\phi)\coloneqq -\widetilde Q_2^{(t)}(\phi),\quad
\widetilde f_i^{(t)}(\phi)\coloneqq -\widetilde Q_{2i}^{(t)}(\phi).
\end{equation*}
Assume that $F$ is bounded below by $F^*$.
In communication round $t$, each client $i$ performs
\begin{equation*}
\phi_{i,0}^{(t)}\coloneqq \phi^{(t)},\qquad
d_{i,-1}^{(t)}\coloneqq 0,\qquad
d_{i,e}^{(t)}\coloneqq \mu d_{i,e-1}^{(t)}+\widetilde g_{i,e}^{(t)},\qquad
\phi_{i,e+1}^{(t)}\coloneqq \phi_{i,e}^{(t)}-\eta d_{i,e}^{(t)},
\end{equation*}
for $e\in\{0,\ldots,E-1\}$, where $\mu\in [0,1)$ denotes the momentum, $E$ denotes the number of local steps, $\eta$ denotes the learning rate, and $\widetilde{g}_{i,e}^{(t)}$ denotes the stochastic gradient. 

\begin{assumption}[Smoothness]\label{assumption:smoothness}
For all $i$ and $t$, $\widetilde f_i^{(t)}$ is $L$-smooth.
\end{assumption}

\begin{assumption}[Stochastic gradient]\label{assumption:stochastic_gradient}
For all $i$, $t$, and $e$, $\widetilde g_{i,e}^{(t)}$ is unbiased with bounded variance $\sigma^2$, \ie $\mathbb{E}\bigl[\widetilde{g}_{i,e}^{(t)}\bigm|\phi_{i,e}^{(t)}\bigr]=\nabla \widetilde{f}_{i}^{(t)}(\phi_{i,e}^{(t)})$ and $\mathbb{E}\bigl[\bigl\|\widetilde{g}_{i,e}^{(t)}-\nabla \widetilde{f}_{i}^{(t)}(\phi_{i,e}^{(t)})\bigr\|^2\bigm|\phi_{i,e}^{(t)}\bigr]\leq \sigma^2$.
\end{assumption}

\begin{assumption}[Inter-client heterogeneity]\label{assumption:inter-client_heterogeneity}
There exist constants $\Gamma_0,\Gamma_1\ge 0$ such that, for all $t$, $\sum_{i=1}^m\rho_i\bigl\|\nabla\widetilde{f}_i^{(t)}(\phi^{(t)})\bigr\|^2\leq \Gamma_0+\Gamma_1\bigl\|\nabla\widetilde{f}^{(t)}(\phi^{(t)})\bigr\|^2$.
\end{assumption}

Assumptions~\ref{assumption:smoothness}--\ref{assumption:inter-client_heterogeneity} are standard in federated optimization~\cite{wang2020fednova,marfoq2021fedem}.

\begin{assumption}[Gradient bridge]\label{assumption:gradient_bridge}
There exist a constant $\Gamma_2\ge 0$ and a non-negative sequence $\{\varepsilon^{(t)}\}_{t\ge 0}$ such that, for all $t$, $\mathbb{E}\bigl[\bigl\|\nabla f^{(t)}(\phi^{(t)})\bigr\|^2\bigr]\leq \Gamma_2\mathbb{E}\bigl[\bigl\|\nabla\widetilde{f}^{(t)}(\phi^{(t)})\bigr\|^2\bigr]+\varepsilon^{(t)}$.
\end{assumption}

\begin{assumption}[Function-value bridge]\label{assumption:function-value_bridge}
There exists a non-negative sequence $\{\delta^{(t)}\}_{t\ge 0}$ such that, for all $t$, $\mathbb{E}\bigl[f^{(t)}(\phi^{(t+1)})-f^{(t)}(\phi^{(t)})\bigr]\leq \mathbb{E}\bigl[\widetilde{f}^{(t)}(\phi^{(t+1)})-\widetilde{f}^{(t)}(\phi^{(t)})\bigr] +\delta^{(t)}$.
\end{assumption}

Assumption~\ref{assumption:gradient_bridge} requires that, along the iterates, the gradient norm of the DRM objective $f^{(t)}$ is bounded by that of its tractable surrogate $\widetilde{f}^{(t)}$ up to an error $\varepsilon^{(t)}$, without imposing any alignment between their gradient directions, while Assumption~\ref{assumption:function-value_bridge} quantifies the one-step mismatch between their function-value changes. 
We further provide empirical support for these assumptions in Sec.~\ref{sec:experimental_results}.

\begin{theorem}\label{theorem:main}
Under Assumptions~\ref{assumption:smoothness}--\ref{assumption:function-value_bridge}, when $\eta = \Theta(1/\sqrt{T})$ and $\eta L a_E < \min\{1/3 , 1/\sqrt{3+8\Gamma_1}\}$, \ours under local SGD with momentum satisfies:
\begin{equation*}
\resizebox{\linewidth}{!}{
$\displaystyle\frac{1}{T}\sum_{t=0}^{T-1}\mathbb{E}\left\|\nabla_\phi F(\zeta^{(t)})\right\|^2
= 
\mathcal{O}\Biggl(
\frac{\Gamma_2(F(\zeta^{(0)})-F^*)}{a_E\sqrt{T}}
+
\frac{\Gamma_2}{a_E}\bar{\delta}_T
+
\frac{\Gamma_2L\sigma^2 s_E}{a_E\sqrt{T}}\sum_{i=1}^m \rho_i^2
+
\frac{\Gamma_2L^2(a_E^2\Gamma_0+\sigma^2 s_E)}{T}
+
\bar{\varepsilon}_T
\Biggr),$}
\end{equation*}
where 
\begin{equation*}
\mu_{e,s}\coloneqq
\frac{1-\mu^{e-s}}{1-\mu},\,\, a_E\coloneqq\sum_{e=0}^{E-1}\mu_{E,e},\,\, s_E\coloneqq\sum_{e=0}^{E-1}\mu_{E,e}^2, \,\,\bar{\delta}_T\coloneqq \frac{1}{\sqrt{T}}\sum_{t=0}^{T-1}\delta^{(t)},\,\, \bar{\varepsilon}_T\coloneqq \frac{1}{T}\sum_{t=0}^{T-1}\varepsilon^{(t)}.
\end{equation*}
\end{theorem}
Crucially, although \ours optimizes the tractable surrogate $\widetilde{f}^{(t)}$, Theorem~\ref{theorem:main} establishes convergence toward a stationary point of the exact profiled objective $F$.
For fixed $E$, if $\sum_{t=0}^{T-1}\delta^{(t)}=\mathcal{O}(1)$ and $\sum_{t=0}^{T-1}\varepsilon^{(t)}=\mathcal{O}(\sqrt{T})$, the surrogate errors do not change the standard $\mathcal{O}(1/\sqrt{T})$ convergence rate.
The bound further characterizes the effect of momentum through the accumulation factors $a_E$ and $s_E$.
When $\mu=0$, $a_E=s_E=E$, recovering the usual local SGD scaling.
As $\mu\to 1$, $a_E=\mathcal{O}(E^2)$ and $s_E=\mathcal{O}(E^3)$.
Consequently, the initial-gap term improves from order $1/(E\sqrt{T})$ to $1/(E^2\sqrt{T})$, whereas the stochastic term grows from order $1/\sqrt{T}$ to $E/\sqrt{T}$, revealing a trade-off between faster optimization and amplified stochastic noise.
The proof is deferred to App.~\ref{app:convergence_proof}.

%% file: sections/experiment.tex
\section{Experiments on Benchmark Datasets}

We conduct experiments on three image classification benchmarks of increasing complexity: FMNIST~\cite{xiao2017fashionmnist}, CIFAR-10~\cite{krizhevsky2009cifar}, and CIFAR-100~\cite{krizhevsky2009cifar}, with dataset details deferred to App.~\ref{app:benchmark_dataset_details}.
To study each heterogeneity type in isolation and assess their joint effect, we construct semi-synthetic FL settings by superimposing controlled inter-client and intra-client heterogeneity onto the original data.
 
\subsection{Training and Evaluation Setup}

\textbf{Baselines.}
We compare \ours with representative FL baselines from seven categories: 
global-model methods (FedAvg~\cite{mcmahan2017fedavg} and FedProx~\cite{li2020fedprox}), a fine-tuning-based method (FedAvgFT~\cite{wang2019ft}), a regularization-based method (Ditto~\cite{li2021ditto}), a cluster-based method (ClusterFL~\cite{ghosh2020ifca}), a representation-learning-based method (FedBABU~\cite{oh2022fedbabu}), mixture-model-based methods (FedEM~\cite{marfoq2021fedem} and FedGMM~\cite{wu2023fedgmm}), and a routing-based method (FedDRM~\cite{wang2026feddrm}). 
Notably, only FedGMM, FedDRM, and \ours are capable of client routing by design. 

\textbf{Evaluation metrics.}
Following~\cite{wang2026feddrm}, we evaluate all methods using system and average accuracies. 
System accuracy is measured on the pooled test set across all clients, computed via \eqref{eq:client_routing} and \eqref{eq:local_prediction} for routing-capable methods, and via majority voting otherwise. 
Average accuracy is the mean of local test accuracies weighted by client sample fractions.

\textbf{Model architecture.}
We instantiate the classification encoder $g_{\theta_c}$ with ResNet~\cite{he2016resnet} and the routing encoder $h_{\nu_c}$ with a lightweight CNN~\cite{lecun1998cnn}.
For a fair comparison, all methods use the same classification encoder architecture, and routing-capable methods use the same routing encoder architecture. 
Notably, the model capacity of cluster- and mixture-model-based methods scales linearly with their number of components. 
For \ours, we report two variants to separate algorithmic gains from increased model capacity.
In \ours~($1\times$), the encoders are shared across components, \ie $g_{\theta_c}\equiv g_{\theta}$ and $h_{\nu_c}\equiv h_{\nu}$ for all $c$.
In \ours~($C\times$), each component has its own encoders $g_{\theta_c}$ and $h_{\nu_c}$.
 
\textbf{Dual heterogeneity settings.} 
Since benchmark datasets do not inherently exhibit statistical heterogeneity, we explicitly construct dual heterogeneity based on the mixture representation $P_i^{X,Y}=\sum_c \pi_{ic}P_{ic}^{X,Y}$, following common practice~\cite{tan2023is,wu2023fedgmm,wang2026feddrm}. 
Specifically, we first partition the full dataset across 8 clients via class-wise Dirichlet partitioning~\cite{yurochkin2019bayesian} with concentration parameter $\alpha_{\text{inter}}=1.0$, inducing label shift and unequal client dataset sizes. 
Within each client, we further divide the local dataset into 2 latent components, with the mixing weights $\pi_i$ sampled from a Dirichlet distribution with $\alpha_{\text{intra}}=2.0$. 
To induce component-wise covariate shift in $P_{ic}^{X}$, we apply transformations at two levels: client-level transformations, which introduce inter-client heterogeneity through combinations of color shifts and spatial intensity biases~\cite{wang2026feddrm}; and component-level transformations, which introduce intra-client heterogeneity through color-channel inversion~\cite{wu2023fedgmm}. 
In addition, for each client-component pair, we randomly generate a label permutation~\cite{wu2023fedgmm} to induce component-wise concept shift in $P_{ic}^{Y|X}$. 
Since the type and degree of each heterogeneity are specified by construction, the ground-truth data-generating mechanism is fully known, making these benchmarks well-suited for controlled comparisons. 
See App.~\ref{app:visualization_of_dual_heterogeneity} for illustrative visualizations.

\textbf{Training details.} 
We employ local SGD with momentum as the \texttt{LocalSolver}.
For fine-tuning-based methods, we additionally perform one epoch of local fine-tuning before evaluation. 
For cluster- and mixture-model-based methods, we set the number of components $C$ to 3.
Additional implementation and hyperparameter details are provided in App.~\ref{app:benchmark_training_details}.
 
\subsection{Experimental Results}\label{sec:experimental_results}

\begin{table}[!ht]
\centering
\caption{\textbf{System and average accuracies on benchmark datasets.} Results are reported as mean $\pm$ standard deviation over 3 random seeds. Superscripts (1), (2), and (3) denote the first-, second-, and third-best results, respectively. $1\times$ and $C\times$ denote the model capacity.}
\label{tab:main_results_benchmark}
\resizebox{0.9\textwidth}{!}{
\begin{tabular}{lcccccc}
\toprule
\multirow{2}{*}{Method} & \multicolumn{3}{c}{System Accuracy} & \multicolumn{3}{c}{Average Accuracy} \\
\cmidrule(lr){2-4} \cmidrule(lr){5-7}
                        & FMNIST & CIFAR-10 & CIFAR-100 & FMNIST & CIFAR-10 & CIFAR-100 \\
\midrule
\rowcolor{blue!10}
FedAvg ($1\times$)      & $39.15\pm1.31$ & $35.02\pm0.32$ & $18.17\pm0.53$ & $39.15\pm1.31$ & $35.02\pm0.32$ & $18.17\pm0.53$ \\
\rowcolor{blue!10}
FedProx ($1\times$)     & $39.14\pm1.32$ & $34.94\pm0.36$ & $18.15\pm0.49$ & $39.14\pm1.32$ & $34.94\pm0.36$ & $18.15\pm0.49$ \\
\rowcolor{pink!10}
FedAvgFT ($1\times$)    & $35.42\pm1.90$ & $32.71\pm0.22$ & $16.29\pm0.69$ & $50.55\pm1.04$ & $41.40\pm0.91$ & $23.51\pm0.54$ \\
\rowcolor{pink!10}
Ditto ($1\times$)       & $34.30\pm2.42$ & $33.12\pm0.45$ & $16.02\pm0.69$ & $54.87\pm1.18$ & $43.91\pm0.26$ & $25.06\pm0.54$ \\
\rowcolor{pink!10}
ClusterFL ($C\times$)   & $31.36\pm2.20$ & $28.38\pm2.48$ & $14.68\pm1.59$ & $55.77\pm2.32\mathrlap{^{\scriptscriptstyle (3)}}$ & $45.33\pm1.20$ & $27.31\pm1.26$ \\
\rowcolor{pink!10}
FedBABU ($1\times$)     & $34.27\pm1.19$ & $33.39\pm0.55$ & $16.64\pm0.25$ & $38.36\pm1.76$ & $36.90\pm0.47$ & $19.09\pm0.62$ \\
\rowcolor{pink!10}
FedEM ($C\times$)       & $29.73\pm2.68$ & $26.71\pm2.07$ & $13.69\pm0.62$ & $38.79\pm0.19$ & $38.50\pm0.42$ & $19.98\pm0.44$ \\
\rowcolor{green!10}
FedGMM ($C\times$)      & $31.98\pm1.99$ & $25.36\pm0.48$ & $11.97\pm0.96$ & $38.34\pm0.58$ & $38.86\pm0.42$ & $20.70\pm0.25$ \\
\rowcolor{green!10}
FedDRM ($1\times$)      & $53.68\pm1.01\mathrlap{^{\scriptscriptstyle (3)}}$ & $47.47\pm0.09\mathrlap{^{\scriptscriptstyle (3)}}$ & $27.15\pm0.37\mathrlap{^{\scriptscriptstyle (3)}}$ & $53.69\pm1.01$ & $47.62\pm0.11\mathrlap{^{\scriptscriptstyle (3)}}$ & $27.55\pm0.37\mathrlap{^{\scriptscriptstyle (3)}}$ \\
\midrule
\ours ($1\times$)       & $\mathbf{58.82\pm0.61}\mathrlap{^{\scriptscriptstyle (1)}}$ & $\mathbf{48.39\pm0.21}\mathrlap{^{\scriptscriptstyle (1)}}$ & $\mathbf{28.33\pm0.37}\mathrlap{^{\scriptscriptstyle (1)}}$ & $\mathbf{58.86\pm0.62}\mathrlap{^{\scriptscriptstyle (1)}}$ & $\mathbf{48.59\pm0.22}\mathrlap{^{\scriptscriptstyle (1)}}$ & $\mathbf{28.84\pm0.40}\mathrlap{^{\scriptscriptstyle (1)}}$ \\
\ours ($C\times$)       & $58.23\pm0.17\mathrlap{^{\scriptscriptstyle (2)}}$ & $48.03\pm0.58\mathrlap{^{\scriptscriptstyle (2)}}$ & $27.54\pm0.45\mathrlap{^{\scriptscriptstyle (2)}}$ & $58.32\pm0.18\mathrlap{^{\scriptscriptstyle (2)}}$ & $48.32\pm0.61\mathrlap{^{\scriptscriptstyle (2)}}$ & $28.07\pm0.46\mathrlap{^{\scriptscriptstyle (2)}}$ \\
\bottomrule
\end{tabular}
}
\end{table}
\textbf{Main results.}
Tab.~\ref{tab:main_results_benchmark} yields four main observations. 
First, \ours consistently outperforms all baselines in both metrics across all datasets, regardless of model capacity. 
Second, routing-free personalization methods (red background) trade system accuracy for average accuracy compared with global-model methods, revealing severe model drift.
The substantial recovery in system accuracy achieved by FedDRM and \ours further underscores the necessity of client routing for system-level performance.
Third, the poor performance of mixture-model-based methods indicates that personalizing only the mixing weights is insufficient to handle severe concept shift without component-level personalization.
Fourth, \ours~($1\times$) slightly outperforms \ours~($C\times$), suggesting that under fixed data budgets on these relatively simple tasks, the benefit of increased model capacity does not outweigh the sample dilution from separating encoders per component.
 
\textbf{Impact of dual heterogeneity intensity.}
We further evaluate the robustness of \ours under different levels of inter-client and intra-client heterogeneity on FMNIST.
To isolate their effects, we vary the Dirichlet concentration parameter $\alpha_{\text{inter}}\in\{0.5,1.0,2.0\}$ while fixing $\alpha_{\text{intra}}=2.0$, and vary $\alpha_{\text{intra}}\in\{0.5,1.0,2.0\}$ while fixing $\alpha_{\text{inter}}=1.0$.
As shown in Tab.~\ref{tab:dual_heterogeneity_results}, \ours consistently achieves the best performance across all settings, demonstrating its strong adaptability to dual heterogeneity.

\begin{table*}[!htbp]
\centering
\caption{\textbf{System and average accuracies under varying dual heterogeneity intensities.}}
\label{tab:dual_heterogeneity_results}
\resizebox{0.9\textwidth}{!}{
\begin{tabular}{lcccccccccccc}
\toprule
\multirow{3}{*}{Method} & \multicolumn{6}{c}{Inter-Client Heterogeneity ($\alpha_{\text{inter}}$)} & \multicolumn{6}{c}{Intra-Client Heterogeneity ($\alpha_{\text{intra}}$)} \\
\cmidrule(lr){2-7} \cmidrule(lr){8-13}
 & \multicolumn{3}{c}{System Accuracy} & \multicolumn{3}{c}{Average Accuracy} & \multicolumn{3}{c}{System Accuracy} & \multicolumn{3}{c}{Average Accuracy} \\
\cmidrule(lr){2-4} \cmidrule(lr){5-7} \cmidrule(lr){8-10} \cmidrule(lr){11-13}
 & 0.5 & 1.0 & 2.0 & 0.5 & 1.0 & 2.0 & 0.5 & 1.0 & 2.0 & 0.5 & 1.0 & 2.0 \\
\midrule
\rowcolor{blue!10}
FedAvg ($1\times$)     & $49.85$ & $39.15$ & $31.37$ & $49.85$ & $39.15$ & $31.37$ & $42.15$ & $37.59$ & $39.15$ & $42.15$ & $37.59$ & $39.15$ \\
\rowcolor{blue!10}
FedProx ($1\times$)    & $49.88$ & $39.14$ & $31.34$ & $49.88$ & $39.14$ & $31.34$ & $42.10$ & $37.58$ & $39.14$ & $42.10$ & $37.58$ & $39.14$ \\
\rowcolor{pink!10}
FedAvgFT ($1\times$)   & $45.87$ & $35.42$ & $27.67$ & $62.45$ & $50.55$ & $45.01$ & $36.72$ & $33.42$ & $35.42$ & $54.11$ & $50.99$ & $50.55$ \\
\rowcolor{pink!10}
Ditto ($1\times$)      & $45.28$ & $34.30$ & $25.56$ & $65.87$ & $54.87$ & $49.36$ & $34.26$ & $31.42$ & $34.30$ & $58.22$ & $55.44$ & $54.87$ \\
\rowcolor{pink!10}
ClusterFL ($C\times$)  & $40.77$ & $31.36$ & $23.96$ & $68.80$ & $55.77\mathrlap{^{\scriptscriptstyle (3)}}$ & $52.41$ & $34.79$ & $31.22$ & $31.36$ & $61.48\mathrlap{^{\scriptscriptstyle (3)}}$ & $57.66\mathrlap{^{\scriptscriptstyle (3)}}$ & $55.77\mathrlap{^{\scriptscriptstyle (3)}}$ \\
\rowcolor{pink!10}
FedBABU ($1\times$)    & $42.93$ & $34.27$ & $27.82$ & $50.14$ & $38.36$ & $31.45$ & $37.16$ & $35.09$ & $34.27$ & $40.70$ & $39.21$ & $38.36$ \\
\rowcolor{pink!10}
FedEM ($C\times$)      & $37.08$ & $29.73$ & $25.70$ & $49.68$ & $38.79$ & $31.57$ & $33.43$ & $28.88$ & $29.73$ & $41.78$ & $39.47$ & $38.79$ \\
\rowcolor{green!10}
FedGMM ($C\times$)     & $34.33$ & $31.98$ & $24.57$ & $49.22$ & $38.34$ & $31.83$ & $28.92$ & $27.90$ & $31.98$ & $44.08$ & $38.59$ & $38.34$ \\
\rowcolor{green!10}
FedDRM ($1\times$)     & $69.03\mathrlap{^{\scriptscriptstyle (3)}}$ & $53.68\mathrlap{^{\scriptscriptstyle (3)}}$ & $52.51\mathrlap{^{\scriptscriptstyle (3)}}$ & $69.05\mathrlap{^{\scriptscriptstyle (3)}}$ & $53.69$ & $52.51\mathrlap{^{\scriptscriptstyle (3)}}$ & $60.25\mathrlap{^{\scriptscriptstyle (3)}}$ & $55.52\mathrlap{^{\scriptscriptstyle (3)}}$ & $53.68\mathrlap{^{\scriptscriptstyle (3)}}$ & $60.28$ & $55.53$ & $53.69$ \\
\midrule
\ours ($1\times$)      & $\mathbf{74.20}\mathrlap{^{\scriptscriptstyle (1)}}$ & $\mathbf{58.82}\mathrlap{^{\scriptscriptstyle (1)}}$ & $55.68\mathrlap{^{\scriptscriptstyle (2)}}$ & $\mathbf{74.24}\mathrlap{^{\scriptscriptstyle (1)}}$ & $\mathbf{58.86}\mathrlap{^{\scriptscriptstyle (1)}}$ & $55.70\mathrlap{^{\scriptscriptstyle (2)}}$ & $\mathbf{62.83}\mathrlap{^{\scriptscriptstyle (1)}}$ & $\mathbf{63.10}\mathrlap{^{\scriptscriptstyle (1)}}$ & $\mathbf{58.82}\mathrlap{^{\scriptscriptstyle (1)}}$ & $\mathbf{62.89}\mathrlap{^{\scriptscriptstyle (1)}}$ & $\mathbf{63.20}\mathrlap{^{\scriptscriptstyle (1)}}$ & $\mathbf{58.86}\mathrlap{^{\scriptscriptstyle (1)}}$ \\
\ours ($C\times$)      & $73.15\mathrlap{^{\scriptscriptstyle (2)}}$ & $58.23\mathrlap{^{\scriptscriptstyle (2)}}$ & $\mathbf{56.57}\mathrlap{^{\scriptscriptstyle (1)}}$ & $73.21\mathrlap{^{\scriptscriptstyle (2)}}$ & $58.32\mathrlap{^{\scriptscriptstyle (2)}}$ & $\mathbf{56.65}\mathrlap{^{\scriptscriptstyle (1)}}$ & $62.48\mathrlap{^{\scriptscriptstyle (2)}}$ & $61.33\mathrlap{^{\scriptscriptstyle (2)}}$ & $58.23\mathrlap{^{\scriptscriptstyle (2)}}$ & $62.58\mathrlap{^{\scriptscriptstyle (2)}}$ & $61.40\mathrlap{^{\scriptscriptstyle (2)}}$ & $58.32\mathrlap{^{\scriptscriptstyle (2)}}$ \\
\bottomrule
\end{tabular}
}
\end{table*}

\begin{figure}[!ht]
\centering
\includegraphics[width=0.95\linewidth]{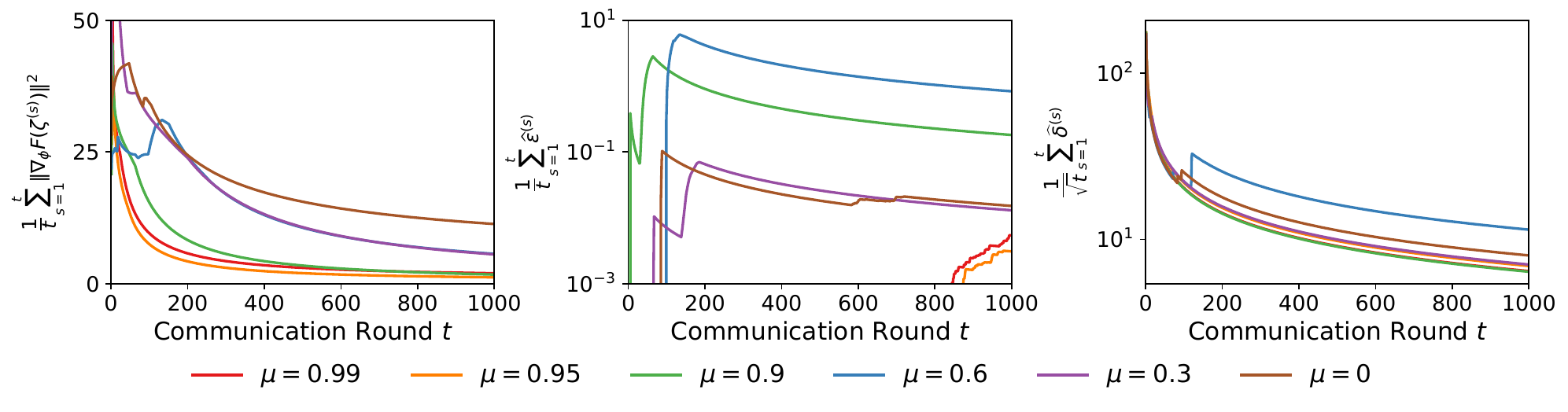}
\vspace{-0.5\baselineskip}
\caption{\textbf{Impact of momentum $\mu$.} From left to right, the panels show the running average of $\|\nabla_\phi F(\zeta^{(t)})\|^2$ and the normalized cumulative sums of $\widehat{\varepsilon}^{(t)}$ and $\widehat{\delta}^{(t)}$.}
\label{fig:impact_of_mu}
\end{figure}

\textbf{Impact of $\mu$.} 
We study the impact of momentum $\mu$ under a constant learning rate on FMNIST.
To align the empirical evaluation with our theory, we report the running average of $\|\nabla_\phi F(\zeta^{(t)})\|^2$, together with the normalized cumulative sums of the estimated bridge errors $\widehat{\delta}^{(t)}$ and $\widehat{\varepsilon}^{(t)}$.
See detailed estimation procedure in App.~\ref{app:estimation_of_the_bridge_errors}.
Since computing these quantities requires evaluating the profiled objective $F$ and the DRM objective $f^{(t)}$, both of which depend on the Lagrange multipliers $\{\lambda_{ic}\}$, we numerically solve the nonlinear system in \eqref{eq:Lagrange_multipliers_system} for $\{\lambda_{ic}\}$ at each communication round.
Fig.~\ref{fig:impact_of_mu} yields three observations.
First, the averaged gradient norm consistently decreases for all tested values of $\mu$, suggesting that \ours empirically converges toward a stationary point of the profiled objective $F$.
Second, as $\mu$ increases, the convergence first improves and then deteriorates, with the best performance achieved at $\mu=0.95$, consistent with the momentum-induced trade-off predicted by our theory.
Third, the bridge errors are empirically well controlled: their normalized cumulative estimates remain small or decrease during training.
This provides empirical support for the asymptotic behaviors $\bar{\delta}_T=o(1)$ and $\bar{\varepsilon}_T=o(1)$, thereby validating the soundness of Assumptions~\ref{assumption:gradient_bridge} and~\ref{assumption:function-value_bridge}.

\begin{wrapfigure}[10]{r}{0.28\textwidth}
\centering
\includegraphics[width=\linewidth]{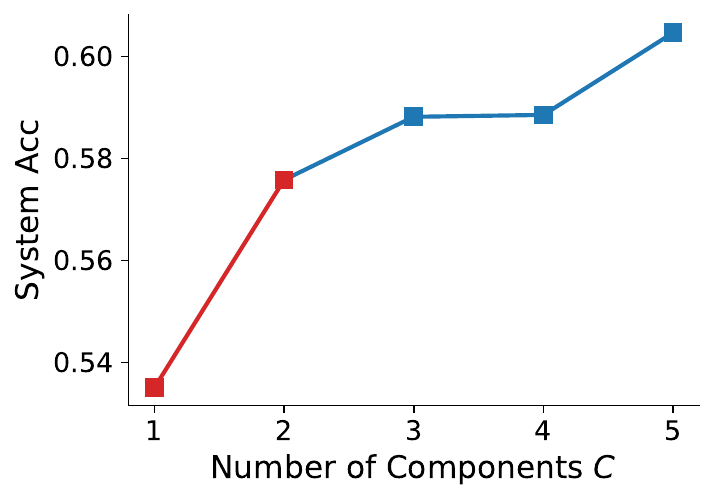}
\caption{\textbf{Impact of $C$.}}
\label{fig:C}
\end{wrapfigure}
\textbf{Impact of $C$.} 
We study the impact of the number of components $C$ on FMNIST. 
To largely isolate the effect of varying $C$ from model capacity changes, we employ a shared encoder across all components so that increasing $C$ only adds lightweight component-specific heads. 
As shown in Fig.~\ref{fig:C}, increasing $C$ from 1 to 2 yields the largest gain, supporting the benefit of mixture modeling under dual heterogeneity.
Further increasing $C$ brings smaller but consistent improvements, suggesting that mild over-specification can provide a more flexible approximation to the predictive distribution, consistent with theoretical results on over-specified mixture-of-experts models~\cite{ho2022moe,nguyen2023moe}.

\section{Experiment on Real Medical Dataset}
We complement the controlled experiments on benchmark datasets with a case study on Fed-ISIC2019~\cite{du_terrail2022flamby}, a real-world medical dataset where dual heterogeneity arises naturally from varying imaging protocols, skewed class proportions, and diverse pathological features.
Fed-ISIC2019 comprises 23,247 RGB dermoscopic images collected from 6 clinical centers for an 8-class skin-lesion classification task, naturally forming a 6-client FL system where each client corresponds to one clinical center. 
Since no heterogeneity is artificially imposed, this natural setting serves two purposes: validating the \emph{necessity} of jointly modeling dual heterogeneity in real-world FL applications and verifying the \emph{effectiveness} of our semiparametric model against parametric alternatives. The training details are provided in App.~\ref{app:isic2019_training_details}.

\begin{figure}[!ht]
\centering
\includegraphics[width=0.9\linewidth]{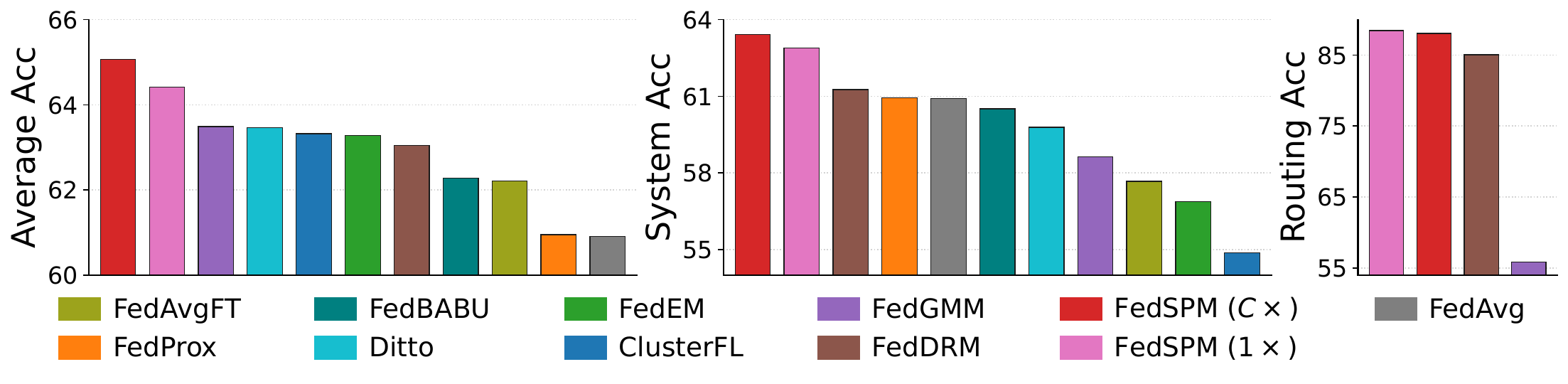}
\vspace{-0.5\baselineskip}
\caption{\textbf{Average, system, and routing accuracies on Fed-ISIC2019.}}
\label{fig:isic2019_results}
\end{figure}

Fig.~\ref{fig:isic2019_results} yields four main observations. 
First, \ours achieves the best performance across all three metrics, showing its practical feasibility in real-world medical scenarios.
Second, mixture-model-based methods generally outperform methods that do not explicitly model intra-client heterogeneity, confirming the presence of intra-client heterogeneity in real-world medical data and validating the necessity of jointly modeling dual heterogeneity.
Third, \ours substantially improves routing accuracy over FedGMM, highlighting the limitation of parametric GMMs and validating the effectiveness of our semiparametric model. 
Finally, in contrast to the benchmark results, \ours~($C\times$) outperforms \ours~($1\times$), suggesting that on this medically complex dataset, the benefit of expanded model capacity for capturing diverse pathological features outweighs the sample dilution effect.

%% file: sections/conclusion.tex
\section{Conclusion}

This paper presents \ours, a semiparametric mixture framework that extends the routing-prediction FL paradigm from inter-client heterogeneity to dual heterogeneity.
By combining DRM-based EL with mixture modeling, \ours leverages latent intra-client structure to improve both routing and prediction while balancing model flexibility with effective information sharing. 
A federated EM procedure further enables practical optimization in the client-server setting.
Empirical results on controlled benchmarks and real-world medical data show that \ours outperforms existing approaches in both routing and prediction performance under complex dual heterogeneity.
Overall, this work paves the way for expertise-aware FL systems that transform dual heterogeneity from an obstacle into a source of system-level intelligence.

%% file: sections/appendix.tex
\section{Related Work on DRM and EL}\label{app:related_work}
The DRM, first introduced by~\citet{anderson1979drm}, provides a statistical framework for modeling multiple related populations by formulating their densities as ratios with respect to a shared baseline distribution.
It is highly flexible and subsumes commonly used parametric families such as the binomial, gamma, and normal distributions~\cite{kay1987transformations}. 
As a semiparametric model, the DRM does not impose parametric assumptions on the baseline distribution, which can instead be handled nonparametrically via EL~\cite{owen1990empirical}. 
This integration gives rise to the DRM-based EL approach, which has garnered substantial attention over the past few decades~\cite{qin1993empirical,qin1999empirical,zhang2000quantile,zhang2002assessing}.
Moreover, \citet{keziou2008empirical} formally established the equivalence between the maximum DRM-based EL estimators and the corresponding dual EL estimators. 
Since the dual EL admits an analytical form and is computationally tractable, it significantly facilitates the implementation and application of DRM-based EL methods~\cite{chen2013quantile,cai2017hypothesis,zhang2026neyman}.

Recently, \citet{wang2026feddrm} first introduced DRM-based EL into FL and established a routing-prediction FL paradigm, thereby enabling explicit client routing together with personalized representation learning.
However, this paradigm treats each client as internally homogeneous.
Since intra-client heterogeneity often arises in real-world FL scenarios, ignoring such latent structure degrades both routing and prediction performance, motivating our semiparametric mixture extension.

\section{The Federated EM Algorithm}\label{app:algorithm}

\begin{algorithm*}[!h]
\caption{\ours}
\label{algo:fedspm}

\For{$t=0,1,\dots, T-1$} {
    Server broadcasts $(\theta^{(t)},\gamma^{(t)}, \xi^{(t)}, \nu^{(t)})$ to all clients
    
    \For{client $i\in [m]$ in parallel} {
        Compute $\{w_{ijc}^{(t)}\}_{j\in[n_i],c\in[C]}$ via \eqref{eq:posterior} 
        
        Send $\{\tau_{ic}=\sum_{j=1}^{n_i}w_{ijc}^{(t)}\}_{c\in[C]}$ to the server
    }

    Server broadcasts $\{\tau_{ic}\}_{i\in[m],c\in[C]}$ to all clients

    \For{client $i\in [m]$ in parallel} {
        Update $\{\pi_{ic}^{(t+1)}\}_{c\in[C]}$ via \eqref{eq:update_pi}

        $(\alpha_i^{(t+1)},\beta_i^{(t+1)}, \theta_i^{(t+1)})\leftarrow \texttt{LocalSolver}(\alpha_i^{(t)}, \beta_i^{(t)}, \theta^{(t)})$

        $(\gamma_i^{(t+1)},\xi_i^{(t+1)}, \nu_i^{(t+1)})\leftarrow \texttt{LocalSolver}(\gamma^{(t)}, \xi^{(t)}, \nu^{(t)})$

        Client $i$ sends $(\theta_i^{(t+1)}, \gamma_i^{(t+1)}, \xi_i^{(t+1)}, \nu_i^{(t+1)})$ to the server
    }

    Server updates $(\theta^{(t+1)}, \gamma^{(t+1)}, \xi^{(t+1)}, \nu^{(t+1)})\leftarrow \sum_{i=1}^m\rho_i(\theta_i^{(t+1)}, \gamma_i^{(t+1)}, \xi_i^{(t+1)}, \nu_i^{(t+1)})$
}
\end{algorithm*}

\section{Mathematical Details}

\subsection{Derivation of the Log-EL}\label{app:log-EL}
The observed data contain only $(x_{ij},y_{ij})$, while the component assignment $z_{ij}$ is latent.
Therefore, the observed-data log-EL is obtained from the augmented distribution $P_i^{X,Y,Z}$ by marginalizing over all possible values of $Z$:
\begin{equation*}
\begin{split}
\ell(\zeta, G)
&\coloneqq
\sum_{i,j}
\log
\sum_c
\underbrace{P_i^{X,Y,Z}(dx_{ij},\{y_{ij}\},\{c\})}_{\text{complete-data contribution for }Z=c}\\
&=
\sum_{i,j}
\log
\sum_c
\pi_{ic}
P_{ic}^{X,Y}(dx_{ij},\{y_{ij}\}) \\
&=
\sum_{i,j}
\log
\sum_c
\pi_{ic}
P_{ic}^{Y|X}(\{y_{ij}\}\mid x_{ij})
P_{ic}^{X}(dx_{ij}) \\
&=
\sum_{i,j}
\log
\sum_c
\pi_{ic}
P_{ic}^{Y|X}(\{y_{ij}\}\mid x_{ij})
\frac{dP_{ic}^X}{dG}(x_{ij})G(dx_{ij}) \\
&=
\sum_{i,j}
\log
\sum_c
\pi_{ic}
P_{ic}^{Y|X}(\{y_{ij}\}\mid x_{ij})
\frac{dP_{ic}^X}{dG}(x_{ij})r_{ij}\\
&=\sum_{i,j} \log r_{ij} + \sum_{i,j} \log \biggl[ \sum_{c} \pi_{ic} \exp \left( \gamma_{ic} + \xi_{ic}^\top h_{\nu_c}(x_{ij}) \right) P_{ic}^{Y|X}(\{y_{ij}\}\mid x_{ij})\biggr].
\end{split}
\end{equation*}

\subsection{Derivation of the Optimal Baseline Weights}\label{app:optimal_G}

For fixed $\zeta$, maximizing $\ell$ over $G$ reduces to maximizing $\sum_{i,j}\log r_{ij}$ subject to the constraints in \eqref{eq:constraints}. 
The corresponding Lagrangian is:
\begin{equation*}
    \resizebox{\linewidth}{!}{
    $\displaystyle \mathcal{L}\coloneqq \sum_{i,j}\log r_{ij}
-n\kappa\biggl[\sum_{i,j}r_{ij}-1\biggr]
-n\sum_{i',c}\lambda_{i'c}\sum_{i,j}
\left[
\exp\left(\gamma_{i'c}+\xi_{i'c}^{\top}h_{\nu_c}(x_{ij})\right)-1
\right]r_{ij}.$
    }
\end{equation*}
Setting $\partial\mathcal{L}/\partial r_{ij}=0$ gives
\begin{equation*}
r_{ij}^{-1}  
-n\kappa
-n\sum_{i',c}\lambda_{i'c}
\left[
\exp\left(\gamma_{i'c}+\xi_{i'c}^{\top}h_{\nu_c}(x_{ij})\right)-1
\right]=0.
\end{equation*}
Multiplying the above equation by $r_{ij}$, summing over $(i,j)$, and applying the constraints in \eqref{eq:constraints} gives $\kappa=1$, which yields
\begin{equation*}
r_{ij}^*(\zeta)=\frac{1}{n}
\biggl\{
1+
\sum_{i',c}\lambda_{i'c}
\left[
\exp\left(\gamma_{i'c}+\xi_{i'c}^{\top}h_{\nu_c}(x_{ij})\right)-1
\right]
\biggr\}^{-1}.
\end{equation*}
Substituting $r_{ij}^*(\zeta)$ back into \eqref{eq:constraints} gives the following nonlinear system for $\{\lambda_{ic}\}$:
\begin{equation*}
\sum_{i,j}
\frac{
\exp\left(\gamma_{i'c}+\xi_{i'c}^{\top}h_{\nu_c}(x_{ij})\right)-1
}{
1+
\sum_{i'',c'}\lambda_{i''c'}
\left[
\exp\left(\gamma_{i''c'}+\xi_{i''c'}^{\top}h_{\nu_{c'}}(x_{ij})\right)-1
\right]
}=0,
\quad \forall i'\in[m],\ c\in[C].
\end{equation*}

\subsection{Derivation of the Complete-Data Profile Log-EL}\label{app:complete-data_profile_log-EL}

Unlike the observed-data log-EL, the complete-data log-EL treats the latent component indicators $\{z_{ijc}\}$ as observed.
Thus, it uses the corresponding complete-data contribution from $P_i^{X,Y,Z}$ directly, without marginalizing over $Z$:
\begin{equation*}
\begin{aligned}
\ell^{\text{c}}(\zeta, G)
&\coloneqq
\sum_{i,j,c}z_{ijc}\log P_i^{X,Y,Z}(dx_{ij},\{y_{ij}\},\{c\})\\
&=
\sum_{i,j,c}z_{ijc}\log\Bigl[\pi_{ic}P_{ic}^{X,Y}(dx_{ij},\{y_{ij}\})\Bigr]\\
&=
\sum_{i,j,c}z_{ijc}\log\Bigl[\pi_{ic}P_{ic}^{Y|X}(\{y_{ij}\}\mid x_{ij})P_{ic}^{X}(dx_{ij})\Bigr]\\
&=
\sum_{i,j,c}z_{ijc}\log\Bigl[\pi_{ic}P_{ic}^{Y|X}(\{y_{ij}\}\mid x_{ij})\frac{dP_{ic}^X}{dG}(x_{ij})G(dx_{ij})\Bigr]\\
&=
\sum_{i,j,c}z_{ijc}\log\Bigl[\pi_{ic}\frac{dP_{ic}^X}{dG}(x_{ij})r_{ij} P_{ic}^{Y|X}(\{y_{ij}\}\mid x_{ij}) \Bigr]\\
&=\sum_{i,j}\log r_{ij}+\sum_{i,j,c}z_{ijc}\Bigl[\log P_{ic}^{Y|X}(\{y_{ij}\}\mid x_{ij})+\gamma_{ic}+\xi_{ic}^\top h_{\nu_c}(x_{ij})+\log \pi_{ic} \Bigr].
\end{aligned}
\end{equation*}
Since maximizing $\ell^{\text{c}}$ over $G$ for fixed $\zeta$ also reduces to maximizing $\sum_{i,j}\log r_{ij}$ subject to the constraints in \eqref{eq:constraints}, $\ell^{\text{c}}$ shares the same optimal baseline weights as $\ell$. 
Thus, substituting $r_{ij}^*(\zeta)$ back into $\ell^{\text{c}}(\zeta,G)$ yields the complete-data profile log-EL:
\begin{equation*}
\begin{aligned}
p\ell^{\text{c}}(\zeta)&\coloneqq \sup_G\ell^{\text{c}}(\zeta,G)\\
&=\sum_{i,j}\log r_{ij}^*(\zeta)+\sum_{i,j,c}z_{ijc}\biggl[\log P_{ic}^{Y|X}(\{y_{ij}\}\mid x_{ij})+\gamma_{ic}+\xi_{ic}^\top h_{\nu_c}(x_{ij})+\log \pi_{ic} \biggr].
\end{aligned}
\end{equation*}

\subsection{EM Algorithm Details}\label{app:EM_algorithm_details}
The EM algorithm is an iterative procedure that starts from an initial estimate $\zeta^{(0)}$. At iteration $t$, given the current parameter estimate $\zeta^{(t)}$, the algorithm alternates between the following two steps.

\textbf{\textit{E-step.}} 
The latent component assignments $\{z_{ij}\}$ are estimated by their posterior responsibilities, which serve as soft guesses computed from the observed data under the current parameter estimate:
\begin{equation*}
w_{ijc}^{(t)} \coloneqq \mathbb{E}\left[z_{ijc}\middle|\mathcal{D};\zeta^{(t)}\right] =P_i^{Z|X,Y}(\{c\}\mid x_{ij},y_{ij};\zeta^{(t)}).
\end{equation*}
The second equality follows because $z_{ijc}$ is the indicator of the event $\{Z=c\}$.
Applying Bayes' rule with prior $\pi_{ic}^{(t)}$ and likelihood $P_{ic}^{X,Y}(dx_{ij},\{y_{ij}\};\zeta^{(t)})$ gives
\begin{equation*}
w_{ijc}^{(t)}
=
\frac{
    \pi_{ic}^{(t)}P_{ic}^{X,Y}(dx_{ij},\{y_{ij}\};\zeta^{(t)})
}{
    \sum_{c'}
    \pi_{ic'}^{(t)}P_{ic'}^{X,Y}(dx_{ij},\{y_{ij}\};\zeta^{(t)})
}.
\end{equation*}
Using the factorization of $P_{ic}^{X,Y}$ into $P_{ic}^{Y|X}$ and $P_{ic}^{X}$ gives
\begin{equation*}
w_{ijc}^{(t)}
=
\frac{
    \pi_{ic}^{(t)}
    P_{ic}^{Y|X}(\{y_{ij}\}\mid x_{ij};\zeta^{(t)})
    P_{ic}^{X}(dx_{ij};\zeta^{(t)})
}{
    \sum_{c'}
    \pi_{ic'}^{(t)}
    P_{ic'}^{Y|X}(\{y_{ij}\}\mid x_{ij};\zeta^{(t)})
    P_{ic'}^{X}(dx_{ij};\zeta^{(t)})
}.
\end{equation*}
Under the EL representation of $G$, $P_{ic}^{X}(dx_{ij})=dP_{ic}^{X}/dG(x_{ij})r_{ij}$.
Therefore,
\begin{equation*}
w_{ijc}^{(t)}
=
\frac{
    \pi_{ic}^{(t)}
    P_{ic}^{Y|X}(\{y_{ij}\}\mid x_{ij};\zeta^{(t)})
    \frac{dP_{ic}^{X}}{dG}(x_{ij};\zeta^{(t)})r_{ij}
}{
    \sum_{c'}
    \pi_{ic'}^{(t)}
    P_{ic'}^{Y|X}(\{y_{ij}\}\mid x_{ij};\zeta^{(t)})
    \frac{dP_{ic'}^{X}}{dG}(x_{ij};\zeta^{(t)})r_{ij}
}.
\end{equation*}
Finally, substituting the DRM density ratio and canceling the common factor $r_{ij}$ yields
\begin{equation*}
w_{ijc}^{(t)}
=
\frac{
P_{ic}^{Y|X}(\{y_{ij}\}\mid x_{ij};\zeta^{(t)})
\exp\left(\gamma_{ic}^{(t)}+(\xi_{ic}^{(t)})^\top h_{\nu_c^{(t)}}(x_{ij})\right)
\pi_{ic}^{(t)}
}{
\sum_{c'}
P_{ic'}^{Y|X}(\{y_{ij}\}\mid x_{ij};\zeta^{(t)})
\exp\left(\gamma_{ic'}^{(t)}+(\xi_{ic'}^{(t)})^\top h_{\nu_{c'}^{(t)}}(x_{ij})\right)
\pi_{ic'}^{(t)}
}.
\end{equation*}
Then, the resulting profile $Q$-function is the conditional expectation of the complete-data profile log-EL with respect to $\{z_{ij}\}$:
\begin{equation*}
    \begin{aligned}
Q^{(t)}(\zeta)
&\coloneqq
\mathbb{E}\left[p\ell^{\text{c}}(\zeta)\middle|\mathcal{D};\zeta^{(t)}\right] \\
&=
Q_1^{(t)}(\alpha,\beta,\theta)
+
Q_2^{(t)}(\gamma,\xi,\nu)
+
Q_3^{(t)}(\pi),
\end{aligned}
\end{equation*}
where
\begin{equation*}
    \begin{aligned}
Q_1^{(t)}(\alpha,\beta,\theta) \coloneqq& \sum_{i,j,c}w_{ijc}^{(t)}\log P_{ic}^{Y|X}(\{y_{ij}\}\mid x_{ij}),\\
Q_2^{(t)}(\gamma, \xi, \nu) \coloneqq& \sum_{i,j,c}w_{ijc}^{(t)} \left[\gamma_{ic} + \xi_{ic}^\top h_{\nu_c}(x_{ij})\right]\\ &- \sum_{i,j}\log \biggl(1 + \sum_{i',c} \lambda_{i'c} \left[\exp\left(\gamma_{i'c} + \xi_{i'c}^\top h_{\nu_c}(x_{ij})\right) - 1\right]\biggr),\\
Q_3^{(t)}(\pi) \coloneqq& \sum_{i,j,c}w_{ijc}^{(t)}\log\pi_{ic}.
\end{aligned}
\end{equation*}

\textbf{\textit{M-step.}}
Instead of directly maximizing the profile log-EL, the M-step maximizes $Q^{(t)}$ with respect to $\zeta$, \ie 
\begin{equation*}
    \zeta^{(t+1)}=\arg\max Q^{(t)}(\zeta).
\end{equation*}
Here we focus on the analytic update of the mixing weights $\pi$, which solves
\begin{equation*}
    \begin{aligned}
\max_{\pi}\quad
&\sum_{ijc}
w_{ijc}^{(t)}\log \pi_{ic} \\
\mathrm{s.t.}\quad
&\sum_{c} \pi_{ic}=1,\qquad i\in [m],\\
&\pi_{ic}\ge 0,\qquad i\in [m],\quad c\in [C].
\end{aligned}
\end{equation*}
Using the method of Lagrange multipliers, the Lagrangian is
\begin{equation*}
    \mathcal{L}_\pi
\coloneqq
\sum_{ijc}
w_{ijc}^{(t)}\log \pi_{ic}
+
\sum_{i}
\iota_i
\left(\sum_{c}\pi_{ic}-1\right).
\end{equation*}
At the optimum, we require
\begin{equation*}
    \frac{\partial\mathcal{L}_\pi}{\partial \pi_{ic}}=0,\quad\frac{\partial\mathcal{L}_\pi}{\partial \iota_i}=0, \quad i\in[m],\quad c\in[C].
\end{equation*}
Solving these equations gives
\begin{equation*}
    \pi_{ic}^{(t+1)}
=
\frac{1}{n_i}
\sum_{j=1}^{n_i}w_{ijc}^{(t)}.
\end{equation*}

\subsection{Derivation of the Analytical Form of the Lagrange Multipliers}\label{app:analytical_form_of_lagrange_multipliers}

We derive the analytical form of $\{\lambda_{ic}\}$ from the first-order optimality condition of $Q_2^{(t)}$. 
For each $i_0\in [m]$ and $c_0\in [C]$, taking the derivative with respect to $\gamma_{i_0,c_0}$ gives
\begin{equation*}
    \frac{\partial Q_2^{(t)}}{\partial \gamma_{i_0,c_0}} =\sum_{j=1}^{n_{i_0}}w_{i_0,j,c_0}^{(t)} -\lambda_{i_0,c_0}\sum_{i,j} \frac{ \exp\left(\gamma_{i_0,c_0}+\xi_{i_0,c_0}^{\top}h_{\nu_{c_0}}(x_{ij})\right) }{ 1+ \sum_{i',c}\lambda_{i'c} \left[ \exp\left(\gamma_{i'c}+\xi_{i'c}^{\top}h_{\nu_c}(x_{ij})\right)-1 \right] },
\end{equation*}
where the term involving $\partial\lambda_{i_0,c_0}/\partial\gamma_{i_0,c_0}$ vanishes due to \eqref{eq:Lagrange_multipliers_system}. 
By \eqref{eq:constraints} and \eqref{eq:optimal_G}, we further have
\begin{equation*}
    \sum_{i,j}
\frac{
\exp\left(\gamma_{i_0,c_0}+\xi_{i_0,c_0}^{\top}h_{\nu_{c_0}}(x_{ij})\right)
}{
1+ \sum_{i',c}\lambda_{i'c} \left[ \exp\left(\gamma_{i'c}+\xi_{i'c}^{\top}h_{\nu_c}(x_{ij})\right)-1 \right] 
} 
=n\sum_{i,j}
\exp\left(\gamma_{i_0,c_0}+\xi_{i_0,c_0}^{\top}h_{\nu_{c_0}}(x_{ij})\right)r_{ij}^*
=n.
\end{equation*}
Therefore,
\begin{equation*}
    \frac{\partial Q_2^{(t)}}{\partial \gamma_{i_0,c_0}} =\sum_{j=1}^{n_{i_0}}w_{i_0,j,c_0}^{(t)} -n\lambda_{i_0,c_0}.
\end{equation*}
Setting the derivative to zero yields
\begin{equation*}
    \lambda_{i_0,c_0}=\frac{1}{n}\sum_{j=1}^{n_{i_0}}w_{i_0,j,c_0}^{(t)}.
\end{equation*}

\subsection{Convergence Proofs}\label{app:convergence_proof}

The proof proceeds in four steps. First, we unroll the local momentum recursion and rewrite the aggregated update as a single descent step driven by a weighted effective gradient. The weights are the momentum accumulation factors $\{\mu_{E,e}\}_{e=0}^{E-1}$, which make the dependence on $\mu$ explicit through $a_E$ and $s_E$. Second, we decompose this effective gradient into the averaged exact gradient, the local-update drift, and the stochastic noise. Lemmas~\ref{lemma:stochastic_noise_bound} and~\ref{lemma:gradient_drift_bound} bound the stochastic noise and drift terms, respectively. Third, Lemma~\ref{lemma:one_round_descent} establishes a one-round descent estimate for the tractable surrogate $\widetilde f^{(t)}$. Finally, we combine Assumptions~\ref{assumption:gradient_bridge} and~\ref{assumption:function-value_bridge} with the EM bridge arguments in Lemmas~\ref{lemma:EM_function-value_bridge} and~\ref{lemma:EM_gradient_bridge} to transfer the descent to the profiled log-EL and telescope over the outer rounds.

\subsubsection{Additional Notations}

For the convergence proof, define the stochastic gradient noise as
\begin{equation*}
\omega_{i,e}^{(t)}\coloneqq
\widetilde g_{i,e}^{(t)}-
\nabla\widetilde f_i^{(t)}(\phi_{i,e}^{(t)}).
\end{equation*}
Let $\mathcal F^{(-1)}$ be the trivial $\sigma$-algebra. 
For $t\ge0$, define
\begin{equation*}
\mathcal{F}^{(t)}\coloneqq
\sigma\left(
\omega_{i,e}^{(s)}:
0\le s\le t,
\ i\in[m],
\ 0\le e\le E-1
\right).
\end{equation*}
For each client $i$ and local step $e$, define
\begin{equation*}
\mathcal{F}_{i,-1}^{(t)}\coloneqq\mathcal{F}^{(t-1)},\quad
\mathcal{F}_{i,e}^{(t)}\coloneqq
\mathcal{F}^{(t-1)}\vee
\sigma\left(
\omega_{i,0}^{(t)},\dots,\omega_{i,e}^{(t)}
\right).
\end{equation*}

\subsubsection{Key Lemmas}

\begin{lemma}[EM function-value bridge]\label{lemma:EM_function-value_bridge}
At communication round $t$, for any $\zeta$, we have
\begin{equation*}
F(\zeta)-F(\zeta^{(t)})
\le
f^{(t)}(\phi)-f^{(t)}(\phi^{(t)}).
\end{equation*}
\end{lemma}

\begin{proof}
We first prove the evidence lower bound (ELBO) for the profiled log-EL:
\begin{equation*}
    p\ell(\zeta)
= \sum_{i,j}\log r_{ij}^*(\zeta)+
\sum_{i,j}
\log
\biggl[
\sum_{c}
\pi_{ic}
\exp\left(
\gamma_{ic}
+
\xi_{ic}^{\top}h_{\nu_c}(x_{ij})
\right)
P_{ic}^{Y|X}(\{y_{ij}\}\mid x_{ij})
\biggr].
\end{equation*}
For each sample $(i,j)$, using $\sum_c w_{ijc}^{(t)}=1$, Jensen's inequality gives
\begin{equation*}
\begin{aligned}
p\ell(\zeta)
&=\sum_{i,j}\log r_{ij}^*(\zeta)+
\sum_{i,j}
\log
\biggl[
\sum_c
w_{ijc}^{(t)}
\frac{
\pi_{ic}
\exp\left(
\gamma_{ic}
+
\xi_{ic}^{\top}h_{\nu_c}(x_{ij})
\right)
P_{ic}^{Y|X}(\{y_{ij}\}\mid x_{ij})
}{
w_{ijc}^{(t)}
}
\biggr] \\
&\ge \sum_{i,j}\log r_{ij}^*(\zeta)+
\sum_{i,j,c}
w_{ijc}^{(t)}
\log
\frac{
\pi_{ic}
\exp\left(
\gamma_{ic}
+
\xi_{ic}^{\top}h_{\nu_c}(x_{ij})
\right)
P_{ic}^{Y|X}(\{y_{ij}\}\mid x_{ij})
}{
w_{ijc}^{(t)}
}\\
&=Q^{(t)}(\zeta)
-
\sum_{i,j,c}
w_{ijc}^{(t)}\log w_{ijc}^{(t)} .
\end{aligned}
\end{equation*}
By \eqref{eq:posterior}, when $\zeta=\zeta^{(t)}$, the Jensen inequality above becomes tight.
Hence,
\begin{equation*}
p\ell(\zeta^{(t)})
=
Q^{(t)}(\zeta^{(t)})
-
\sum_{i,j,c}
w_{ijc}^{(t)}\log w_{ijc}^{(t)} .
\end{equation*}
Subtracting this equality from the previous lower bound gives
\begin{equation*}
p\ell(\zeta)-p\ell(\zeta^{(t)})
\ge
Q^{(t)}(\zeta)-Q^{(t)}(\zeta^{(t)}).
\end{equation*}
By the monotone ascent property of the standard EM/GEM updates for $Q_1^{(t)}$ and $Q_3^{(t)}$, we have
\begin{equation*}
Q^{(t)}(\zeta)-Q^{(t)}(\zeta^{(t)})
\ge
Q_2^{(t)}(\phi)-Q_2^{(t)}(\phi^{(t)}).
\end{equation*}
Therefore, since $F=-p\ell$ and $f^{(t)}=-Q_2^{(t)}$, we obtain
\begin{equation*}
F(\zeta)-F(\zeta^{(t)})
\le
f^{(t)}(\phi)-f^{(t)}(\phi^{(t)}).
\end{equation*}
\end{proof}

\begin{lemma}[EM gradient bridge]\label{lemma:EM_gradient_bridge}
At communication round $t$, we have
\begin{equation*}
\nabla_\phi F(\zeta^{(t)})
=
\nabla f^{(t)}(\phi^{(t)}).
\end{equation*}
\end{lemma}

\begin{proof}
By the ELBO for the profiled log-EL established in the proof of Lemma~\ref{lemma:EM_function-value_bridge}, for any $\zeta$,
\begin{equation*}
p\ell(\zeta)-p\ell(\zeta^{(t)})
\ge
Q^{(t)}(\zeta)-Q^{(t)}(\zeta^{(t)}).
\end{equation*}
Equivalently,
\begin{equation*}
p\ell(\zeta)-Q^{(t)}(\zeta)
\ge
p\ell(\zeta^{(t)})-Q^{(t)}(\zeta^{(t)}).
\end{equation*}
Therefore, $p\ell(\zeta)-Q^{(t)}(\zeta)$ attains a local minimum at $\zeta^{(t)}$. 
Since the involved functions are differentiable, its $\phi$-gradient vanishes at $\zeta^{(t)}$, namely
\begin{equation*}
\nabla_\phi p\ell(\zeta^{(t)})
=
\nabla_\phi Q^{(t)}(\zeta^{(t)}).
\end{equation*}
By \eqref{eq:Q_decomposition}, only $Q_2^{(t)}$ depends on $\phi$. 
Thus,
\begin{equation*}
\nabla_\phi Q^{(t)}(\zeta^{(t)})
=
\nabla Q_2^{(t)}(\phi^{(t)}).
\end{equation*}
Since $F=-p\ell$ and $f^{(t)}=-Q_2^{(t)}$, we obtain
\begin{equation*}
\nabla_\phi F(\zeta^{(t)})
=
-\nabla_\phi p\ell(\zeta^{(t)})
=
-\nabla Q_2^{(t)}(\phi^{(t)})
=
\nabla f^{(t)}(\phi^{(t)}).
\end{equation*}
\end{proof}

\begin{lemma}[Momentum unrolling identity]\label{lemma:momentum_unrolling_identity}
At communication round $t$, we have
\begin{equation*}
\phi^{(t+1)}
=
\phi^{(t)}-\eta \widetilde g^{(t)},
\end{equation*}
where
\begin{equation*}
\widetilde g^{(t)}
\coloneqq
\sum_{i=1}^m \rho_i
\sum_{e=0}^{E-1}\mu_{E,e}\widetilde g_{i,e}^{(t)}.
\end{equation*}
\end{lemma}

\begin{proof}
By unrolling the momentum recursion, for every $e=0,\dots,E-1$,
\begin{equation*}
d_{i,e}^{(t)}
=
\mu d_{i,e-1}^{(t)}+\widetilde g_{i,e}^{(t)}
=
\sum_{s=0}^{e}\mu^{e-s}\widetilde g_{i,s}^{(t)}.
\end{equation*}
Therefore,
\begin{equation*}
\begin{aligned}
\phi_{i,E}^{(t)}-\phi^{(t)}
&=
-\eta\sum_{e=0}^{E-1}d_{i,e}^{(t)} \\
&=
-\eta\sum_{e=0}^{E-1}\sum_{s=0}^{e}\mu^{e-s}\widetilde g_{i,s}^{(t)} \\
&=
-\eta\sum_{s=0}^{E-1}\left(\sum_{e=s}^{E-1}\mu^{e-s}\right)\widetilde g_{i,s}^{(t)} \\
&=
-\eta\sum_{s=0}^{E-1}\mu_{E,s}\widetilde g_{i,s}^{(t)}.
\end{aligned}
\end{equation*}
Using the aggregation rule and $\sum_{i=1}^m \rho_i=1$ gives
\begin{equation*}
\begin{aligned}
\phi^{(t+1)}-\phi^{(t)}
&=
\sum_{i=1}^m \rho_i\left(\phi_{i,E}^{(t)}-\phi^{(t)}\right) \\
&=
-\eta\sum_{i=1}^m \rho_i\sum_{e=0}^{E-1}\mu_{E,e}\widetilde g_{i,e}^{(t)} \\
&=
-\eta \widetilde g^{(t)}.
\end{aligned}
\end{equation*}
\end{proof}

\begin{lemma}[Gradient decomposition]\label{lemma:gradient_decomposition}
At communication round $t$, we have
\begin{equation*}
\widetilde g^{(t)}
=
a_E\nabla \widetilde f^{(t)}(\phi^{(t)})+D^{(t)}+\Xi^{(t)},
\end{equation*}
where
\begin{equation*}
D^{(t)}
\coloneqq
\sum_{i=1}^m \rho_i
\sum_{e=0}^{E-1}
\mu_{E,e}\Bigl[
\nabla \widetilde f_i^{(t)}(\phi_{i,e}^{(t)})-\nabla \widetilde f_i^{(t)}(\phi^{(t)})
\Bigr],
\end{equation*}
and
\begin{equation*}
\Xi^{(t)}
\coloneqq
\sum_{i=1}^m \rho_i
\sum_{e=0}^{E-1}
\mu_{E,e}\omega_{i,e}^{(t)}.
\end{equation*}
\end{lemma}

\begin{proof}
By definition,
\[
\begin{aligned}
\widetilde g^{(t)}
&=
\sum_{i=1}^m \rho_i\sum_{e=0}^{E-1}\mu_{E,e}\widetilde g_{i,e}^{(t)} \\
&=
\sum_{i=1}^m \rho_i\sum_{e=0}^{E-1}\mu_{E,e}\nabla \widetilde f_i^{(t)}(\phi_{i,e}^{(t)})
+\Xi^{(t)}.
\end{aligned}
\]
Adding and subtracting $\nabla \widetilde f_i^{(t)}(\phi^{(t)})$ gives
\[
\begin{aligned}
\widetilde g^{(t)}
&=
\sum_{i=1}^m \rho_i\sum_{e=0}^{E-1}\mu_{E,e}\nabla \widetilde f_i^{(t)}(\phi^{(t)})
+D^{(t)}+\Xi^{(t)} \\
&=
\left(\sum_{e=0}^{E-1}\mu_{E,e}\right)\sum_{i=1}^m \rho_i\nabla \widetilde f_i^{(t)}(\phi^{(t)})
+D^{(t)}+\Xi^{(t)} \\
&=
a_E\nabla \widetilde f^{(t)}(\phi^{(t)})+D^{(t)}+\Xi^{(t)}.
\end{aligned}
\]
\end{proof}

\begin{lemma}[Stochastic noise bound]\label{lemma:stochastic_noise_bound}
Under Assumption~\ref{assumption:stochastic_gradient}, at communication round $t$, we have
\begin{equation*}
\mathbb{E}\left[\left\|\Xi^{(t)}\right\|^2 \,\middle|\, \mathcal F^{(t-1)}\right]
\le
\sigma^2 s_E\sum_{i=1}^m \rho_i^2.
\end{equation*}
\end{lemma}

\begin{proof}
For each client $i$, $\{\omega_{i,e}^{(t)}\}_{e=0}^{E-1}$ is a martingale difference sequence with
\begin{equation*}
\mathbb E\left[\omega_{i,e}^{(t)} \,\middle|\, \mathcal F_{i,e-1}^{(t)}\right]=0,
\qquad
\mathbb E\left[\left\|\omega_{i,e}^{(t)}\right\|^2 \,\middle|\, \mathcal F_{i,e-1}^{(t)}\right]\le \sigma^2.
\end{equation*}
Thus the within-client cross terms vanish, and
\begin{equation*}
\begin{aligned}
\mathbb E\left[
\left\|
\sum_{e=0}^{E-1}\mu_{E,e}\omega_{i,e}^{(t)}
\right\|^2
\,\middle|\, \mathcal F^{(t-1)}
\right]
&=
\sum_{e=0}^{E-1}\mu_{E,e}^2
\mathbb E\left[\left\|\omega_{i,e}^{(t)}\right\|^2 \,\middle|\, \mathcal F^{(t-1)}\right] \\
&\le
\sigma^2 s_E.
\end{aligned}
\end{equation*}
Since the accumulated noises are conditionally independent across clients given $\mathcal F^{(t-1)}$, the cross-client terms also vanish. 
Hence, 
\begin{equation*}
\begin{aligned}
\mathbb E\left[\left\|\Xi^{(t)}\right\|^2 \,\middle|\, \mathcal F^{(t-1)}\right]
&=
\sum_{i=1}^m \rho_i^2
\mathbb E\left[
\left\|
\sum_{e=0}^{E-1}\mu_{E,e}\omega_{i,e}^{(t)}
\right\|^2
\,\middle|\, \mathcal F^{(t-1)}
\right] \\
&\le
\sigma^2 s_E\sum_{i=1}^m \rho_i^2.
\end{aligned}
\end{equation*}
\end{proof}

\begin{lemma}[Gradient drift bound]\label{lemma:gradient_drift_bound}
Under Assumptions~\ref{assumption:smoothness}--\ref{assumption:inter-client_heterogeneity}, at communication round $t$, if $\eta La_E<1/\sqrt{3}$, we have
\begin{equation*}
\mathbb E\left[\left\|D^{(t)}\right\|^2 \,\middle|\, \mathcal F^{(t-1)}\right]
\le
\frac{2a_E^2L^2\eta^2}{1-3a_E^2L^2\eta^2}
\left[
a_E^2\left(\Gamma_0+\Gamma_1\left\|\nabla \widetilde f^{(t)}(\phi^{(t)})\right\|^2\right)
+s_E\sigma^2
\right].
\end{equation*}
\end{lemma}

\begin{proof}
By Jensen's inequality and $L$-smoothness,
\begin{equation*}
\begin{aligned}
\left\|D^{(t)}\right\|^2
&\le
a_E
\sum_{i=1}^m \rho_i
\sum_{e=0}^{E-1}
\mu_{E,e}
\left\|
\nabla \widetilde f_i^{(t)}(\phi_{i,e}^{(t)})-\nabla \widetilde f_i^{(t)}(\phi^{(t)})
\right\|^2 \\
&\le
a_E L^2
\sum_{i=1}^m \rho_i
\sum_{e=0}^{E-1}
\mu_{E,e}
\left\|\phi_{i,e}^{(t)}-\phi^{(t)}\right\|^2.
\end{aligned}
\end{equation*}
Set
\begin{equation*}
M_e^{(t)}
\coloneqq
\sum_{i=1}^m \rho_i
\mathbb E\left[\left\|\phi_{i,e}^{(t)}-\phi^{(t)}\right\|^2 \,\middle|\, \mathcal F^{(t-1)}\right],
\qquad
M_{\max}^{(t)}
\coloneqq
\max_{0\le e\le E-1}M_e^{(t)}.
\end{equation*}
Then
\begin{equation*}
\mathbb E\left[\left\|D^{(t)}\right\|^2 \,\middle|\, \mathcal F^{(t-1)}\right]
\le
a_E^2 L^2 M_{\max}^{(t)}.
\end{equation*}
It remains to bound $M_{\max}^{(t)}$. For $e\ge1$,
\begin{equation*}
\phi_{i,e}^{(t)}-\phi^{(t)}
=
-\eta\sum_{s=0}^{e-1}\mu_{e,s}\widetilde g_{i,s}^{(t)}.
\end{equation*}
Using
\begin{equation*}
\widetilde g_{i,s}^{(t)}
=
\nabla \widetilde f_i^{(t)}(\phi^{(t)})
+
\Bigl[\nabla \widetilde f_i^{(t)}(\phi_{i,s}^{(t)})-\nabla \widetilde f_i^{(t)}(\phi^{(t)})\Bigr]
+
\omega_{i,s}^{(t)},
\end{equation*}
martingale cancellation, Young's inequality with equal weights, and $L$-smoothness imply
\begin{equation*}
\begin{aligned}
M_e^{(t)}
&\le
2\eta^2
\left(\sum_{s=0}^{e-1}\mu_{e,s}\right)^2
\sum_{i=1}^m \rho_i
\left\|\nabla \widetilde f_i^{(t)}(\phi^{(t)})\right\|^2 \\
&\quad+
3\eta^2L^2
\left(\sum_{s=0}^{e-1}\mu_{e,s}\right)
\sum_{s=0}^{e-1}\mu_{e,s}M_s^{(t)}
+
2\eta^2\sigma^2
\sum_{s=0}^{e-1}\mu_{e,s}^2.
\end{aligned}
\end{equation*}
By Assumption~\ref{assumption:inter-client_heterogeneity},
\begin{equation*}
\sum_{i=1}^m \rho_i
\left\|\nabla \widetilde f_i^{(t)}(\phi^{(t)})\right\|^2
\le
\Gamma_0+\Gamma_1\left\|\nabla \widetilde f^{(t)}(\phi^{(t)})\right\|^2.
\end{equation*}
Since
\begin{equation*}
\sum_{s=0}^{e-1}\mu_{e,s}\le a_E,
\qquad
\sum_{s=0}^{e-1}\mu_{e,s}^2\le s_E,
\qquad
M_s^{(t)}\le M_{\max}^{(t)},
\end{equation*}
we get
\begin{equation*}
\begin{aligned}
M_e^{(t)}
&\le
2\eta^2 a_E^2
\left(\Gamma_0+\Gamma_1\left\|\nabla \widetilde f^{(t)}(\phi^{(t)})\right\|^2\right)
+
2\eta^2 s_E \sigma^2
+
3\eta^2L^2 a_E^2 M_{\max}^{(t)}.
\end{aligned}
\end{equation*}
Taking the maximum over $e$ gives
\begin{equation*}
M_{\max}^{(t)}
\le
\frac{2\eta^2\left[
a_E^2\left(\Gamma_0+\Gamma_1\left\|\nabla \widetilde f^{(t)}(\phi^{(t)})\right\|^2\right)
+s_E\sigma^2
\right]}{1-3a_E^2L^2\eta^2}.
\end{equation*}
Substituting this into the bound for $D^{(t)}$ proves the result.
\end{proof}

\begin{lemma}[One-round descent]\label{lemma:one_round_descent}
Under Assumptions~\ref{assumption:smoothness}--\ref{assumption:inter-client_heterogeneity}, at communication round $t$, if $\eta L a_E<\min\{1/3,1/\sqrt{3+8\Gamma_1}\}$, we have
\begin{equation*}
\begin{aligned}
\mathbb E\Bigl[
\widetilde f^{(t)}&(\phi^{(t+1)})-\widetilde f^{(t)}(\phi^{(t)})
\,\Bigm|\, \mathcal F^{(t-1)}
\Bigr] \\
&\qquad\le
-\frac{a_E\eta}{4}\left\|\nabla \widetilde f^{(t)}(\phi^{(t)})\right\|^2
+
s_E L\eta^2\sigma^2\sum_{i=1}^m \rho_i^2
+
3a_E L^2\eta^3\left(a_E^2\Gamma_0+s_E\sigma^2\right).
\end{aligned}
\end{equation*}
\end{lemma}

\begin{proof}
By Lemma~\ref{lemma:momentum_unrolling_identity} and $L$-smoothness,
\begin{equation*}
\widetilde f^{(t)}(\phi^{(t+1)})-\widetilde f^{(t)}(\phi^{(t)})
\le
-\eta
\left\langle
\nabla \widetilde f^{(t)}(\phi^{(t)}),
\widetilde g^{(t)}
\right\rangle
+
\frac{L\eta^2}{2}\left\|\widetilde g^{(t)}\right\|^2.
\end{equation*}
Using Lemma~\ref{lemma:gradient_decomposition}, expanding the square, and taking conditional expectation, the terms involving $\langle\nabla \widetilde f^{(t)}(\phi^{(t)}),\Xi^{(t)}\rangle$ vanish. Since
\begin{equation*}
2\langle D^{(t)},\Xi^{(t)}\rangle
\le
\left\|D^{(t)}\right\|^2+\left\|\Xi^{(t)}\right\|^2,
\end{equation*}
we obtain
\begin{equation*}
\begin{aligned}
&\mathbb E\left[
\widetilde f^{(t)}(\phi^{(t+1)})-\widetilde f^{(t)}(\phi^{(t)})
\,\middle|\, \mathcal F^{(t-1)}
\right] \\
&\qquad\le
-a_E\eta\left\|\nabla \widetilde f^{(t)}(\phi^{(t)})\right\|^2
+
\frac{L a_E^2\eta^2}{2}\left\|\nabla \widetilde f^{(t)}(\phi^{(t)})\right\|^2 \\
&\qquad\quad+
\left(-\eta+L a_E\eta^2\right)
\mathbb E\left[
\left\langle
\nabla \widetilde f^{(t)}(\phi^{(t)}),
D^{(t)}
\right\rangle
\,\middle|\, \mathcal F^{(t-1)}
\right] \\
&\qquad\quad+
L\eta^2
\mathbb E\left[\left\|D^{(t)}\right\|^2 \,\middle|\, \mathcal F^{(t-1)}\right]
+
L\eta^2
\mathbb E\left[\left\|\Xi^{(t)}\right\|^2 \,\middle|\, \mathcal F^{(t-1)}\right].
\end{aligned}
\end{equation*}
Since $a_E L\eta<1/3$,
\begin{equation*}
-\eta+L a_E\eta^2=-\eta(1-a_E L\eta)\le0.
\end{equation*}
Using
\begin{equation*}
-\left\langle
\nabla \widetilde f^{(t)}(\phi^{(t)}),
D^{(t)}
\right\rangle
\le
\frac{a_E}{2}\left\|\nabla \widetilde f^{(t)}(\phi^{(t)})\right\|^2
+
\frac{1}{2a_E}\left\|D^{(t)}\right\|^2,
\end{equation*}
we have
\begin{equation*}
\begin{aligned}
&\mathbb E\left[
\widetilde f^{(t)}(\phi^{(t+1)})-\widetilde f^{(t)}(\phi^{(t)})
\,\middle|\, \mathcal F^{(t-1)}
\right] \\
&\qquad\le
-\frac{a_E\eta}{2}\left\|\nabla \widetilde f^{(t)}(\phi^{(t)})\right\|^2
+
\left(\frac{\eta}{2a_E}+\frac{L\eta^2}{2}\right)
\mathbb E\left[\left\|D^{(t)}\right\|^2 \,\middle|\, \mathcal F^{(t-1)}\right] \\
&\qquad\quad+
L\eta^2
\mathbb E\left[\left\|\Xi^{(t)}\right\|^2 \,\middle|\, \mathcal F^{(t-1)}\right].
\end{aligned}
\end{equation*}
By Lemmas~\ref{lemma:stochastic_noise_bound} and~\ref{lemma:gradient_drift_bound},
\begin{equation*}
\begin{aligned}
&\left(\frac{\eta}{2a_E}+\frac{L\eta^2}{2}\right)
\mathbb E\left[\left\|D^{(t)}\right\|^2 \,\middle|\, \mathcal F^{(t-1)}\right] \\
&\qquad\le
\frac{a_E L^2\eta^3(1+a_E L\eta)}{1-3a_E^2L^2\eta^2}
\left[
a_E^2\left(\Gamma_0+\Gamma_1\left\|\nabla \widetilde f^{(t)}(\phi^{(t)})\right\|^2\right)
+s_E\sigma^2
\right],
\end{aligned}
\end{equation*}
and
\begin{equation*}
L\eta^2
\mathbb E\left[\left\|\Xi^{(t)}\right\|^2 \,\middle|\, \mathcal F^{(t-1)}\right]
\le
s_E L\eta^2\sigma^2\sum_{i=1}^m \rho_i^2.
\end{equation*}
The drift contribution to the gradient coefficient equals
\begin{equation*}
a_E\eta
\frac{\Gamma_1(a_E L\eta)^2(1+a_E L\eta)}{1-3a_E^2L^2\eta^2}
\left\|\nabla \widetilde f^{(t)}(\phi^{(t)})\right\|^2.
\end{equation*}
The learning rate condition implies
\begin{equation*}
\frac{\Gamma_1(a_E L\eta)^2(1+a_E L\eta)}{1-3a_E^2L^2\eta^2}
\le
\frac14,
\qquad
\frac{1+a_E L\eta}{1-3a_E^2L^2\eta^2}
\le 3.
\end{equation*}
Combining the above inequalities proves the lemma.
\end{proof}

\subsubsection{Proof of Theorem~\ref{theorem:main}}

\begin{proof}
By Lemma~\ref{lemma:one_round_descent} and the tower property,
\begin{equation*}
\begin{aligned}
&\mathbb E\left[
\widetilde f^{(t)}(\phi^{(t+1)})-\widetilde f^{(t)}(\phi^{(t)})
\right] \\
&\qquad\qquad\le
-\frac{a_E\eta}{4}
\mathbb E\left[\left\|\nabla \widetilde f^{(t)}(\phi^{(t)})\right\|^2\right]
+
s_E L\eta^2\sigma^2\sum_{i=1}^m \rho_i^2+
3a_E L^2\eta^3\left(a_E^2\Gamma_0+s_E\sigma^2\right).
\end{aligned}
\end{equation*}
Combining the preceding inequality with Lemma~\ref{lemma:EM_function-value_bridge} and Assumption~\ref{assumption:function-value_bridge}, we obtain
\begin{equation*}
\begin{aligned}
\frac{a_E\eta}{4}
\mathbb E\left[\left\|\nabla \widetilde f^{(t)}(\phi^{(t)})\right\|^2\right]
&\le
\mathbb E\left[F(\zeta^{(t)})-F(\zeta^{(t+1)})\right]
+
\delta^{(t)} \\
&\quad+
s_E L\eta^2\sigma^2\sum_{i=1}^m \rho_i^2
+
3a_E L^2\eta^3\left(a_E^2\Gamma_0+s_E\sigma^2\right).
\end{aligned}
\end{equation*}
Summing over $t=0,\ldots,T-1$ and using $F(\zeta^{(T)})\ge F^*$,
\begin{equation*}
\begin{aligned}
\frac{a_E\eta}{4}
\sum_{t=0}^{T-1}\mathbb E\left[\left\|\nabla \widetilde f^{(t)}(\phi^{(t)})\right\|^2\right]
&\le
F(\zeta^{(0)})-F^*
+
\sum_{t=0}^{T-1}\delta^{(t)} \\
&\quad+
T s_E L\eta^2\sigma^2\sum_{i=1}^m \rho_i^2
+
3T a_E L^2\eta^3\left(a_E^2\Gamma_0+s_E\sigma^2\right).
\end{aligned}
\end{equation*}
Thus,
\begin{equation*}
\begin{aligned}
\frac{1}{T}\sum_{t=0}^{T-1}\mathbb E\left[\left\|\nabla \widetilde f^{(t)}(\phi^{(t)})\right\|^2\right]
&\le
\frac{4\left(F(\zeta^{(0)})-F^*+\sum_{t=0}^{T-1}\delta^{(t)}\right)}{a_E\eta T} \\
&\quad+
\frac{4s_E L\eta\sigma^2}{a_E}\sum_{i=1}^m \rho_i^2
+
12L^2\eta^2\left(a_E^2\Gamma_0+s_E\sigma^2\right).
\end{aligned}
\end{equation*}
By Lemma~\ref{lemma:EM_gradient_bridge} and Assumption~\ref{assumption:gradient_bridge},
\begin{equation*}
\mathbb E\left[\left\|\nabla_\phi F(\zeta^{(t)})\right\|^2\right]
\le
\Gamma_2\,
\mathbb E\left[\left\|\nabla \widetilde f^{(t)}(\phi^{(t)})\right\|^2\right]
+\varepsilon^{(t)}.
\end{equation*}
Averaging over $t$ and substituting the preceding bound gives
\begin{equation*}
\begin{aligned}
\frac{1}{T}\sum_{t=0}^{T-1}\mathbb E\left[\left\|\nabla_\phi F(\zeta^{(t)})\right\|^2\right]
&\le
\frac{4\Gamma_2\left(F(\zeta^{(0)})-F^*+\sum_{t=0}^{T-1}\delta^{(t)}\right)}{a_E\eta T} \\
&\quad+
\frac{4\Gamma_2 s_E L\eta\sigma^2}{a_E}\sum_{i=1}^m \rho_i^2 \\
&\quad+
12\Gamma_2 L^2\eta^2\left(a_E^2\Gamma_0+s_E\sigma^2\right)
+
\frac{1}{T}\sum_{t=0}^{T-1}\varepsilon^{(t)}.
\end{aligned}
\end{equation*}
Taking $\eta=\Theta(1/\sqrt T)$ yields
\begin{equation*}
\begin{aligned}
\frac{1}{T}\sum_{t=0}^{T-1}\mathbb{E}\left\|\nabla_\phi F(\zeta^{(t)})\right\|^2
= 
\mathcal{O}\Biggl(&
\frac{\Gamma_2(F(\zeta^{(0)})-F^*)}{a_E\sqrt{T}}
+
\frac{\Gamma_2}{a_E}\bar{\delta}_T
+
\frac{\Gamma_2L\sigma^2 s_E}{a_E\sqrt{T}}\sum_{i=1}^m \rho_i^2\\
&+
\frac{\Gamma_2L^2(a_E^2\Gamma_0+\sigma^2 s_E)}{T}
+
\bar{\varepsilon}_T
\Biggr),
\end{aligned}
\end{equation*}
which completes the proof.
\end{proof}

\section{Experimental Details}

\subsection{Benchmark Dataset Details}\label{app:benchmark_dataset_details}

We conduct experiments on FMNIST~\cite{xiao2017fashionmnist}, CIFAR-10~\cite{krizhevsky2009cifar}, and CIFAR-100~\cite{krizhevsky2009cifar}. FMNIST contains 70,000 $28\times 28$ grayscale images from 10 fashion classes. CIFAR-10 contains 60,000 $32\times 32$ RGB images from 10 object classes. CIFAR-100 contains 60,000 $32\times 32$ RGB images from 100 object classes, grouped into 20 superclasses. We construct three tasks of increasing complexity: (a) 10-class classification on FMNIST, (b) 10-class classification on CIFAR-10, and (c) 20-class classification using the CIFAR-100 superclasses.

\subsection{Visualization of Dual Heterogeneity}\label{app:visualization_of_dual_heterogeneity}

We provide visualizations of the dual heterogeneity settings in our benchmark experiments. 

To induce covariate shift, images are transformed at both the client and component levels. 
At the component level, component 0 retains the original image, while component 1 applies grayscale inversion for FMNIST and green-channel inversion for CIFAR-10/100. 
At the client level, the 8 clients correspond to the $2^3$ binary combinations of three orthogonal dimensions: color shift (red- vs. blue-dominant), vertical spatial bias (top- vs. bottom-brightened), and horizontal spatial bias (left- vs. right-brightened). 
As shown in Fig.~\ref{fig:transform_visualization}, these hierarchical transformations yield a visually conspicuous covariate shift across clients.

Fig.~\ref{fig:fmnist_inter} illustrates the impact of $\alpha_{\text{inter}}$ on client label distributions. 
A smaller $\alpha_{\text{inter}}$ produces increasingly skewed class proportions and uneven dataset sizes across clients, corresponding to intensified inter-client heterogeneity. 
Similarly, Fig.~\ref{fig:fmnist_intra} visualizes how $\alpha_{\text{intra}}$ modulates the intra-client mixing weights. 
While a small $\alpha_{\text{intra}}$ concentrates local data into a single dominant component, a larger $\alpha_{\text{intra}}$ yields more balanced mixtures, indicating stronger intra-client heterogeneity.

\begin{figure}[!ht]
\centering
\includegraphics[width=0.95\linewidth]{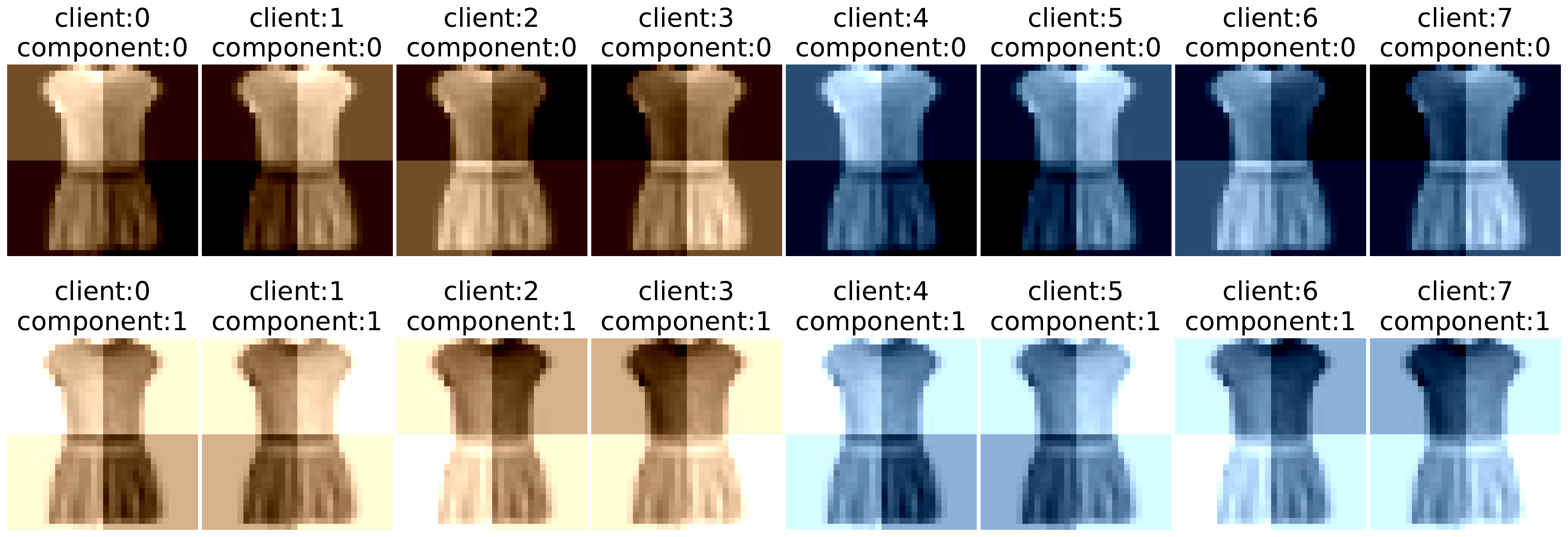}
\caption{Visualization of an FMNIST sample under client- and component-level transformations.}
\label{fig:transform_visualization}
\end{figure}

\begin{figure}[!ht]
\centering
\includegraphics[width=0.95\linewidth]{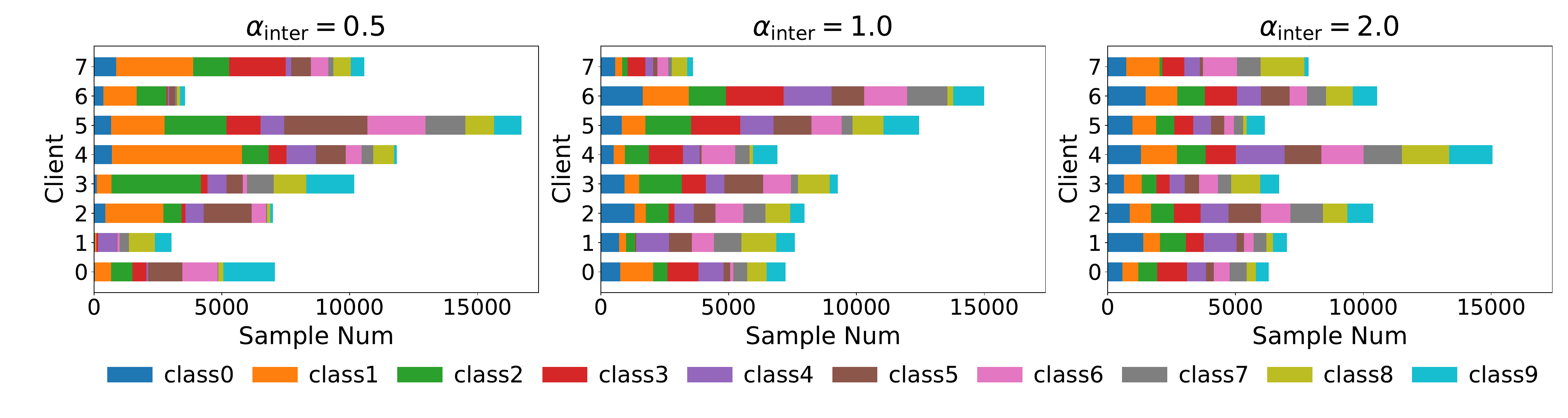}
\caption{Visualization of client label distributions under varying $\alpha_{\text{inter}}$.}
\label{fig:fmnist_inter}
\end{figure}

\begin{figure}[!ht]
\centering
\includegraphics[width=0.95\linewidth]{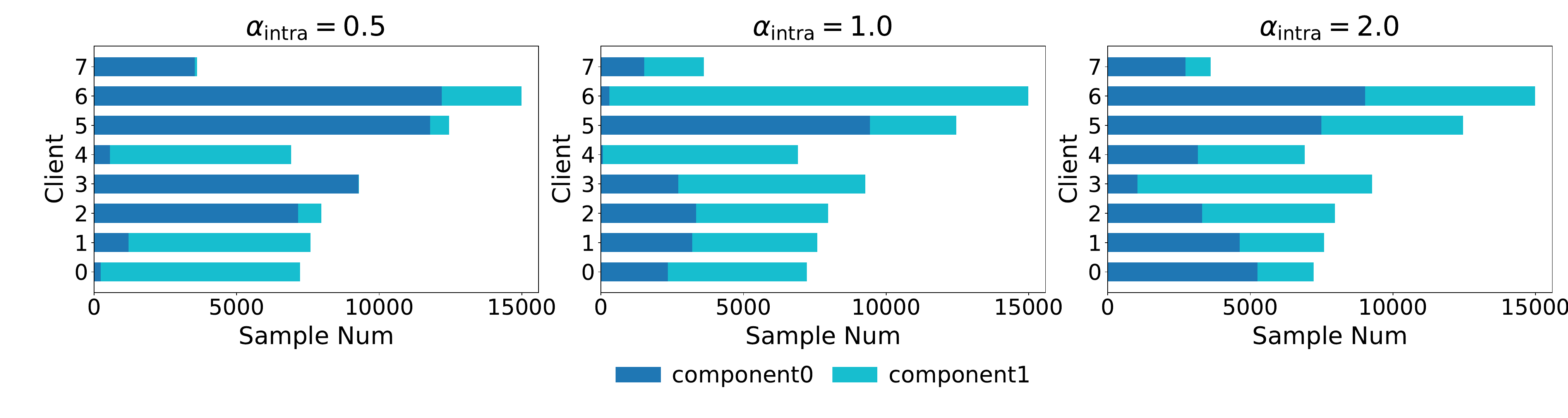}
\caption{Visualization of client mixing weights under varying $\alpha_{\text{intra}}$.}
\label{fig:fmnist_intra}
\end{figure}

\subsection{Training Details on Benchmark Datasets}\label{app:benchmark_training_details}

We employ local SGD with momentum, a batch size of 128, and an initial learning rate of 0.01 with cosine annealing. 
Each client performs 10 local steps per round with momentum $0.9$ for image classification, and 15 local steps with momentum $0.95$ for client routing. 
The encoder output dimensions and total communication rounds are set to 32 and 200 for FMNIST, and 64 and 400 for both CIFAR-10 and CIFAR-100. 
For fine-tuning-based methods, we additionally perform one epoch of local fine-tuning before evaluation. 
For cluster- and mixture-model-based methods, we set the number of components $C$ to 3.

\subsection{Estimation of the Bridge Errors}\label{app:estimation_of_the_bridge_errors}
This section details the empirical estimation of the bridge errors $\varepsilon^{(t)}$ and $\delta^{(t)}$ reported in Fig.~\ref{fig:impact_of_mu}. 
Computing these quantities requires evaluating the DRM objective $f^{(t)}$ and its gradient $\nabla f^{(t)}$, both of which depend on the Lagrange multipliers $\{\lambda_{ic}\}$.
Therefore, at each communication round, we numerically solve the nonlinear system in \eqref{eq:Lagrange_multipliers_system} for $\{\lambda_{ic}\}$ using a Jacobian-based root solver, followed by a trust-region least-squares refinement when necessary.
The resulting multipliers are then used to compute $\{f^{(t)}(\phi^{(t)})\}_{t=0}^{T-1}$ and $\{\nabla f^{(t)}(\phi^{(t)})\}_{t=0}^{T-1}$ along the training trajectory.

\textbf{Estimation of the gradient bridge error $\varepsilon^{(t)}$.}
Assumption~\ref{assumption:gradient_bridge} gives
\begin{equation*}
    \mathbb{E}\left[\left\|\nabla f^{(t)}(\phi^{(t)})\right\|^2\right]\leq \Gamma_2\mathbb{E}\left[\left\|\nabla\widetilde{f}^{(t)}(\phi^{(t)})\right\|^2\right]+\varepsilon^{(t)}.
\end{equation*}
Motivated by this inequality, we first estimate the coefficient $\Gamma_2$ from the ratios between the gradient norms of the DRM objective and its tractable surrogate.
To prevent numerical instability when the surrogate gradient norm vanishes in the denominator, our evaluation is restricted to the valid set:
\begin{equation*}
    \mathcal{T}_{\text{valid}}\coloneqq \left\{t:\left\|\nabla\widetilde{f}^{(t)}(\phi^{(t)})\right\|^2>10^{-9}\right\}.
\end{equation*}
We then estimate $\Gamma_2$ by the 95\% quantile of the valid ratios to ensure robustness against outliers:
\begin{equation*}
    \widehat{\Gamma}_2\coloneqq 
\text{Quantile}_{0.95}\left\{
\frac{
\left\|\nabla f^{(t)}(\phi^{(t)})\right\|^2
}{
\left\|\nabla\widetilde{f}^{(t)}(\phi^{(t)})\right\|^2
}
: t\in\mathcal{T}_{\mathrm{valid}}
\right\}.
\end{equation*}
Finally, given $\widehat{\Gamma}_2$, $\varepsilon^{(t)}$ is computed as the remaining positive gap:
\begin{equation*}
    \widehat{\varepsilon}^{(t)}\coloneqq \max\left(\left\|\nabla f^{(t)}(\phi^{(t)})\right\|^2-\widehat{\Gamma}_2\left\|\nabla \widetilde{f}^{(t)}(\phi^{(t)})\right\|^2,0\right).
\end{equation*}

\textbf{Estimation of the function-value bridge error $\delta^{(t)}$.}
Assumption~\ref{assumption:function-value_bridge} gives
\begin{equation*}
    \mathbb{E}\left[f^{(t)}(\phi^{(t+1)})-f^{(t)}(\phi^{(t)})\right]\leq \mathbb{E}\left[\widetilde{f}^{(t)}(\phi^{(t+1)})-\widetilde{f}^{(t)}(\phi^{(t)})\right] +\delta^{(t)}.
\end{equation*}
Accordingly, we estimate $\delta^{(t)}$ as
\begin{equation*}
    \widehat{\delta}^{(t)}\coloneqq \max\left(\left[f^{(t)}(\phi^{(t+1)})-f^{(t)}(\phi^{(t)})\right]-\left[\widetilde{f}^{(t)}(\phi^{(t+1)})-\widetilde{f}^{(t)}(\phi^{(t)})\right],0\right).
\end{equation*}

\subsection{Training Details on Real Medical Dataset}\label{app:isic2019_training_details}
The training setup for Fed-ISIC2019 largely follows that of the benchmark experiments, with several adjustments. Specifically, the encoder output dimension is set to 256 for image classification and 128 for client routing. The number of communication rounds is set to 400, and the batch size is set to 32.

%% file: main.bbl
\begin{thebibliography}{50}
\providecommand{\natexlab}[1]{#1}
\providecommand{\url}[1]{\texttt{#1}}
\expandafter\ifx\csname urlstyle\endcsname\relax
  \providecommand{\doi}[1]{doi: #1}\else
  \providecommand{\doi}{doi: \begingroup \urlstyle{rm}\Url}\fi

\bibitem[Acar et~al.(2021)Acar, Zhao, Matas, Mattina, Whatmough, and
  Saligrama]{acar2021feddyn}
D.~A.~E. Acar, Y.~Zhao, R.~Matas, M.~Mattina, P.~Whatmough, and V.~Saligrama.
\newblock Federated learning based on dynamic regularization.
\newblock In \emph{International Conference on Learning Representations}, 2021.

\bibitem[Anderson(1979)]{anderson1979drm}
J.~A. Anderson.
\newblock Multivariate logistic compounds.
\newblock \emph{Biometrika}, 66\penalty0 (1):\penalty0 17--26, 1979.

\bibitem[Arivazhagan et~al.(2019)Arivazhagan, Aggarwal, Singh, and
  Choudhary]{arivazhagan2019fedper}
M.~G. Arivazhagan, V.~Aggarwal, A.~K. Singh, and S.~Choudhary.
\newblock Federated learning with personalization layers.
\newblock \emph{arXiv preprint arXiv:1912.00818}, 2019.

\bibitem[Braga et~al.(2008)Braga, Scope, Klaz, Mecca, Spencer, and
  Marghoob]{braga2008melanoma}
J.~C.~T. Braga, A.~Scope, I.~Klaz, P.~Mecca, P.~Spencer, and A.~A. Marghoob.
\newblock Melanoma mimicking seborrheic keratosis: an error of perception
  precluding correct dermoscopic diagnosis.
\newblock \emph{Journal of the American Academy of Dermatology}, 58\penalty0
  (5):\penalty0 875--880, 2008.

\bibitem[Cai et~al.(2017)Cai, Chen, and Zidek]{cai2017hypothesis}
S.~Cai, J.~Chen, and J.~V. Zidek.
\newblock Hypothesis testing in the presence of multiple samples under density
  ratio models.
\newblock \emph{Statistica Sinica}, 27\penalty0 (2):\penalty0 761--783, 2017.

\bibitem[Carrera et~al.(2017)Carrera, Segura, Aguilera, Scalvenzi, Longo,
  Barreiro, Broganelli, Cavicchini, Llambrich, Zaballos, Thomas, Malvehy, Puig,
  and Zalaudek]{carrera2017dermoscopic}
C.~Carrera, S.~Segura, P.~Aguilera, M.~Scalvenzi, C.~Longo, A.~Barreiro,
  P.~Broganelli, S.~Cavicchini, A.~Llambrich, P.~Zaballos, L.~Thomas,
  J.~Malvehy, S.~Puig, and I.~Zalaudek.
\newblock Dermoscopic clues for diagnosing melanomas that resemble seborrheic
  keratosis.
\newblock \emph{JAMA Dermatology}, 153\penalty0 (6):\penalty0 544--551, 2017.

\bibitem[Changchien et~al.(2007)Changchien, Dusza, Agero, Korzenko, Braun,
  Sachs, Usman, Halpern, and Marghoob]{changchien2007age}
L.~Changchien, S.~W. Dusza, A.~L.~C. Agero, A.~J. Korzenko, R.~P. Braun,
  D.~Sachs, M.~H.~U. Usman, A.~C. Halpern, and A.~A. Marghoob.
\newblock Age- and site-specific variation in the dermoscopic patterns of
  congenital melanocytic nevi: an aid to accurate classification and assessment
  of melanocytic nevi.
\newblock \emph{Archives of Dermatology}, 143\penalty0 (8):\penalty0
  1007--1014, 2007.

\bibitem[Chen and Liu(2013)]{chen2013quantile}
J.~Chen and Y.~Liu.
\newblock Quantile and quantile-function estimations under density ratio model.
\newblock \emph{The Annals of Statistics}, 41\penalty0 (3):\penalty0
  1669--1692, 2013.

\bibitem[Dinh et~al.(2020)Dinh, Tran, and Nguyen]{dinh2020pfedme}
C.~T. Dinh, N.~H. Tran, and T.~D. Nguyen.
\newblock Personalized federated learning with moreau envelopes.
\newblock In \emph{Advances in Neural Information Processing Systems}, 2020.

\bibitem[Ghosh et~al.(2020)Ghosh, Chung, Yin, and Ramchandran]{ghosh2020ifca}
A.~Ghosh, J.~Chung, D.~Yin, and K.~Ramchandran.
\newblock An efficient framework for clustered federated learning.
\newblock In \emph{Advances in Neural Information Processing Systems}, 2020.

\bibitem[He et~al.(2016)He, Zhang, Ren, and Sun]{he2016resnet}
K.~He, X.~Zhang, S.~Ren, and J.~Sun.
\newblock Deep residual learning for image recognition.
\newblock In \emph{IEEE Conference on Computer Vision and Pattern Recognition},
  2016.

\bibitem[Ho et~al.(2022)Ho, Yang, and Jordan]{ho2022moe}
N.~Ho, C.-Y. Yang, and M.~I. Jordan.
\newblock Convergence rates for gaussian mixtures of experts.
\newblock \emph{Journal of Machine Learning Research}, 23\penalty0
  (323):\penalty0 1--81, 2022.

\bibitem[Hsu et~al.(2019)Hsu, Qi, and Brown]{hsu2019fedavgm}
H.~Hsu, H.~Qi, and M.~Brown.
\newblock Measuring the effects of non-identical data distribution for
  federated visual classification.
\newblock \emph{arXiv preprint arXiv:1909.06335}, 2019.

\bibitem[Kairouz et~al.(2021)Kairouz, McMahan, Avent, Bellet, Bennis, Bhagoji,
  Bonawitz, Charles, Cormode, Cummings, D'Oliveira, Eichner, Rouayheb, Evans,
  Gardner, Garrett, Gasc{\'o}n, Ghazi, Gibbons, Gruteser, Harchaoui, He, He,
  Huo, Hutchinson, Hsu, Jaggi, Javidi, Joshi, Khodak, Kone{\v c}n{\'y},
  Korolova, Koushanfar, Koyejo, Lepoint, Liu, Mittal, Mohri, Nock,
  {\"O}zg{\"u}r, Pagh, Qi, Ramage, Raskar, Raykova, Song, Song, Stich, Sun,
  Suresh, Tram{\`e}r, Vepakomma, Wang, Xiong, Xu, Yang, Yu, Yu, and
  Zhao]{kairouz2021survey}
P.~Kairouz, H.~B. McMahan, B.~Avent, A.~Bellet, M.~Bennis, A.~N. Bhagoji,
  K.~Bonawitz, Z.~Charles, G.~Cormode, R.~Cummings, R.~G.~L. D'Oliveira,
  H.~Eichner, S.~E. Rouayheb, D.~Evans, J.~Gardner, Z.~Garrett, A.~Gasc{\'o}n,
  B.~Ghazi, P.~B. Gibbons, M.~Gruteser, Z.~Harchaoui, C.~He, L.~He, Z.~Huo,
  B.~Hutchinson, J.~Hsu, M.~Jaggi, T.~Javidi, G.~Joshi, M.~Khodak, J.~Kone{\v
  c}n{\'y}, A.~Korolova, F.~Koushanfar, S.~Koyejo, T.~Lepoint, Y.~Liu,
  P.~Mittal, M.~Mohri, R.~Nock, A.~{\"O}zg{\"u}r, R.~Pagh, H.~Qi, D.~Ramage,
  R.~Raskar, M.~Raykova, D.~Song, W.~Song, S.~U. Stich, Z.~Sun, A.~T. Suresh,
  F.~Tram{\`e}r, P.~Vepakomma, J.~Wang, L.~Xiong, Z.~Xu, Q.~Yang, F.~X. Yu,
  H.~Yu, and S.~Zhao.
\newblock Advances and open problems in federated learning.
\newblock \emph{Foundations and Trends in Machine Learning}, 14\penalty0
  (1--2):\penalty0 1--210, 2021.

\bibitem[Karimireddy et~al.(2020)Karimireddy, Kale, Mohri, Reddi, Stich, and
  Suresh]{karimireddy2020scaffold}
S.~P. Karimireddy, S.~Kale, M.~Mohri, S.~J. Reddi, S.~U. Stich, and A.~T.
  Suresh.
\newblock Scaffold: Stochastic controlled averaging for federated learning.
\newblock In \emph{International Conference on Machine Learning}, 2020.

\bibitem[Kay and Little(1987)]{kay1987transformations}
R.~Kay and S.~Little.
\newblock Transformations of the explanatory variables in the logistic
  regression model for binary data.
\newblock \emph{Biometrika}, 74\penalty0 (3):\penalty0 495--501, 1987.

\bibitem[Keziou and Leoni-Aubin(2008)]{keziou2008empirical}
A.~Keziou and S.~Leoni-Aubin.
\newblock On empirical likelihood for semiparametric two-sample density ratio
  models.
\newblock \emph{Journal of Statistical Planning and Inference}, 138\penalty0
  (4):\penalty0 915--928, 2008.

\bibitem[Krizhevsky(2009)]{krizhevsky2009cifar}
A.~Krizhevsky.
\newblock Learning multiple layers of features from tiny images.
\newblock Technical report, University of Toronto, 2009.

\bibitem[LeCun et~al.(1998)LeCun, Bottou, Bengio, and Haffner]{lecun1998cnn}
Y.~LeCun, L.~Bottou, Y.~Bengio, and P.~Haffner.
\newblock Gradient-based learning applied to document recognition.
\newblock \emph{Proceedings of the IEEE}, 86\penalty0 (11):\penalty0
  2278--2324, 1998.

\bibitem[Li et~al.(2020{\natexlab{a}})Li, Fan, Tse, and Lin]{li2020application}
L.~Li, Y.~Fan, M.~Tse, and K.-Y. Lin.
\newblock A review of applications in federated learning.
\newblock \emph{Computers \& Industrial Engineering}, 149:\penalty0 106854,
  2020{\natexlab{a}}.

\bibitem[Li et~al.(2020{\natexlab{b}})Li, Sahu, Talwalkar, and
  Smith]{li2020survey}
T.~Li, A.~K. Sahu, A.~Talwalkar, and V.~Smith.
\newblock Federated learning: Challenges, methods, and future directions.
\newblock \emph{IEEE Signal Processing Magazine}, 37\penalty0 (3):\penalty0
  50--60, 2020{\natexlab{b}}.

\bibitem[Li et~al.(2020{\natexlab{c}})Li, Sahu, Zaheer, Sanjabi, Talwalkar, and
  Smith]{li2020fedprox}
T.~Li, A.~K. Sahu, M.~Zaheer, M.~Sanjabi, A.~Talwalkar, and V.~Smith.
\newblock Federated optimization in heterogeneous networks.
\newblock In \emph{Proceedings of Machine Learning and Systems},
  2020{\natexlab{c}}.

\bibitem[Li et~al.(2021)Li, Hu, Beirami, and Smith]{li2021ditto}
T.~Li, S.~Hu, A.~Beirami, and V.~Smith.
\newblock Ditto: Fair and robust federated learning through personalization.
\newblock In \emph{International Conference on Machine Learning}, 2021.

\bibitem[Li et~al.(2020{\natexlab{d}})Li, Huang, Yang, Wang, and
  Zhang]{li2020theory}
X.~Li, K.~Huang, W.~Yang, S.~Wang, and Z.~Zhang.
\newblock On the convergence of fedavg on non-iid data.
\newblock In \emph{International Conference on Learning Representations},
  2020{\natexlab{d}}.

\bibitem[Li et~al.(2023)Li, Lin, Shang, and Wu]{li2023fedlaw}
Z.~Li, T.~Lin, X.~Shang, and C.~Wu.
\newblock Revisiting weighted aggregation in federated learning with neural
  networks.
\newblock In \emph{International Conference on Machine Learning}, 2023.

\bibitem[Marfoq et~al.(2021)Marfoq, Neglia, Bellet, Kameni, and
  Vidal]{marfoq2021fedem}
O.~Marfoq, G.~Neglia, A.~Bellet, L.~Kameni, and R.~Vidal.
\newblock Federated multi-task learning under a mixture of distributions.
\newblock In \emph{Advances in Neural Information Processing Systems}, 2021.

\bibitem[McMahan et~al.(2017)McMahan, Moore, Ramage, Hampson, and
  y~Arcas]{mcmahan2017fedavg}
B.~McMahan, E.~Moore, D.~Ramage, S.~Hampson, and B.~A. y~Arcas.
\newblock Communication-efficient learning of deep networks from decentralized
  data.
\newblock In \emph{International Conference on Artificial Intelligence and
  Statistics}, 2017.

\bibitem[Nguyen et~al.(2023)Nguyen, Nguyen, and Ho]{nguyen2023moe}
H.~Nguyen, T.~Nguyen, and N.~Ho.
\newblock Demystifying {Softmax} gating function in {Gaussian} mixture of
  experts.
\newblock In \emph{Advances in Neural Information Processing Systems}, 2023.

\bibitem[Ogier~du Terrail et~al.(2022)Ogier~du Terrail, Ayed, Cyffers,
  Grimberg, He, Loeb, Mangold, Marchand, Marfoq, Mushtaq, Muzellec,
  Philippenko, Silva, Tele\'{n}czuk, Albarqouni, Avestimehr, Bellet,
  Dieuleveut, Jaggi, Karimireddy, Lorenzi, Neglia, Tommasi, and
  Andreux]{du_terrail2022flamby}
J.~Ogier~du Terrail, S.-S. Ayed, E.~Cyffers, F.~Grimberg, C.~He, R.~Loeb,
  P.~Mangold, T.~Marchand, O.~Marfoq, E.~Mushtaq, B.~Muzellec, C.~Philippenko,
  S.~Silva, M.~Tele\'{n}czuk, S.~Albarqouni, S.~Avestimehr, A.~Bellet,
  A.~Dieuleveut, M.~Jaggi, S.~P. Karimireddy, M.~Lorenzi, G.~Neglia,
  M.~Tommasi, and M.~Andreux.
\newblock Flamby: Datasets and benchmarks for cross-silo federated learning in
  realistic healthcare settings.
\newblock In \emph{Advances in Neural Information Processing Systems}, 2022.

\bibitem[Oh et~al.(2022)Oh, Kim, and Yun]{oh2022fedbabu}
J.~Oh, S.~Kim, and S.-Y. Yun.
\newblock Fedbabu: Toward enhanced representation for federated image
  classification.
\newblock In \emph{International Conference on Learning Representations}, 2022.

\bibitem[Owen(1990)]{owen1990empirical}
A.~B. Owen.
\newblock Empirical likelihood ratio confidence regions.
\newblock \emph{The Annals of Statistics}, 18\penalty0 (1):\penalty0 90--120,
  1990.

\bibitem[Qin(1993)]{qin1993empirical}
J.~Qin.
\newblock Empirical likelihood in biased sample problems.
\newblock \emph{The Annals of Statistics}, 21\penalty0 (3):\penalty0
  1182--1196, 1993.

\bibitem[Qin(1999)]{qin1999empirical}
J.~Qin.
\newblock Empirical likelihood ratio based confidence intervals for mixture
  proportions.
\newblock \emph{The Annals of Statistics}, 27\penalty0 (4):\penalty0
  1368--1384, 1999.

\bibitem[Reddi et~al.(2021)Reddi, Charles, Zaheer, Garrett, Rush,
  Kone{\v{c}}n{\'y}, Kumar, and McMahan]{reddi2021fedopt}
S.~J. Reddi, Z.~Charles, M.~Zaheer, Z.~Garrett, K.~Rush, J.~Kone{\v{c}}n{\'y},
  S.~Kumar, and H.~B. McMahan.
\newblock Adaptive federated optimization.
\newblock In \emph{International Conference on Learning Representations}, 2021.

\bibitem[Sattler et~al.(2021)Sattler, M{\"u}ller, and
  Samek]{sattler2021clusterfl}
F.~Sattler, K.-R. M{\"u}ller, and W.~Samek.
\newblock Clustered federated learning: Model-agnostic distributed multi-task
  optimization under privacy constraints.
\newblock \emph{IEEE Transactions on Neural Networks and Learning Systems},
  32\penalty0 (8):\penalty0 3710--3722, 2021.

\bibitem[Stich(2019)]{stich2018localsgd}
S.~U. Stich.
\newblock Local {SGD} converges fast and communicates little.
\newblock In \emph{International Conference on Learning Representations}, 2019.

\bibitem[Tan et~al.(2023)Tan, Chen, Zhuang, Dong, Lyu, and Long]{tan2023is}
Y.~Tan, C.~Chen, W.~Zhuang, X.~Dong, L.~Lyu, and G.~Long.
\newblock Is heterogeneity notorious? taming heterogeneity to handle test-time
  shift in federated learning.
\newblock In \emph{Advances in Neural Information Processing Systems}, 2023.

\bibitem[Wang et~al.(2020)Wang, Liu, Liang, Joshi, and Poor]{wang2020fednova}
J.~Wang, Q.~Liu, H.~Liang, G.~Joshi, and H.~V. Poor.
\newblock Tackling the objective inconsistency problem in heterogeneous
  federated optimization.
\newblock In \emph{Advances in Neural Information Processing Systems}, 2020.

\bibitem[Wang et~al.(2019)Wang, Mathews, Kiddon, Eichner, Beaufays, and
  Ramage]{wang2019ft}
K.~Wang, R.~Mathews, C.~Kiddon, H.~Eichner, F.~Beaufays, and D.~Ramage.
\newblock {Federated Evaluation of On-device Personalization}.
\newblock \emph{arXiv preprint arXiv:1910.10252}, 2019.

\bibitem[Wang et~al.(2026)Wang, Zhang, Zhang, Liu, and Zhang]{wang2026feddrm}
Z.~Wang, X.~Zhang, X.~Zhang, Y.~Liu, and Q.~Zhang.
\newblock Beyond aggregation: Guiding clients in heterogeneous federated
  learning.
\newblock In \emph{International Conference on Learning Representations}, 2026.

\bibitem[Wu et~al.(2023)Wu, Zhang, Yu, Liu, Gu, Zhou, Chen, and
  Cheng]{wu2023fedgmm}
Y.~Wu, S.~Zhang, W.~Yu, Y.~Liu, Q.~Gu, D.~Zhou, H.~Chen, and W.~Cheng.
\newblock Personalized federated learning under mixture of distributions.
\newblock In \emph{International Conference on Machine Learning}, 2023.

\bibitem[Xiao et~al.(2017)Xiao, Rasul, and Vollgraf]{xiao2017fashionmnist}
H.~Xiao, K.~Rasul, and R.~Vollgraf.
\newblock Fashion-mnist: a novel image dataset for benchmarking machine
  learning algorithms.
\newblock \emph{arXiv preprint arXiv:1708.07747}, 2017.

\bibitem[Xu et~al.(2021)Xu, Glicksberg, Su, Walker, Bian, and
  Wang]{xu2021application}
J.~Xu, B.~S. Glicksberg, C.~Su, P.~Walker, J.~Bian, and F.~Wang.
\newblock Federated learning for healthcare informatics.
\newblock \emph{Journal of Healthcare Informatics Research}, 5\penalty0
  (1):\penalty0 1--19, 2021.

\bibitem[Yang et~al.(2019)Yang, Liu, Chen, and Tong]{yang2019federated}
Q.~Yang, Y.~Liu, T.~Chen, and Y.~Tong.
\newblock Federated machine learning: Concept and applications.
\newblock \emph{ACM Transactions on Intelligent Systems and Technology},
  10\penalty0 (2):\penalty0 1--19, 2019.

\bibitem[Yurochkin et~al.(2019)Yurochkin, Agarwal, Ghosh, Greenewald, Hoang,
  and Khazaeni]{yurochkin2019bayesian}
M.~Yurochkin, M.~Agarwal, S.~Ghosh, K.~Greenewald, N.~Hoang, and Y.~Khazaeni.
\newblock Bayesian nonparametric federated learning of neural networks.
\newblock In \emph{International Conference on Machine Learning}, 2019.

\bibitem[Zalaudek et~al.(2006)Zalaudek, Grinschgl, Argenziano, Marghoob, Blum,
  Richtig, Wolf, Fink-Puches, Kerl, Soyer, and
  Hofmann-Wellenhof]{zalaudek2006age}
I.~Zalaudek, S.~Grinschgl, G.~Argenziano, A.~A. Marghoob, A.~Blum, E.~Richtig,
  I.~H. Wolf, R.~Fink-Puches, H.~Kerl, H.~P. Soyer, and R.~Hofmann-Wellenhof.
\newblock Age-related prevalence of dermoscopy patterns in acquired melanocytic
  naevi.
\newblock \emph{British Journal of Dermatology}, 154\penalty0 (2):\penalty0
  299--304, 2006.

\bibitem[Zhang(2000)]{zhang2000quantile}
B.~Zhang.
\newblock Quantile estimation under a two-sample semi-parametric model.
\newblock \emph{Bernoulli}, 6\penalty0 (3):\penalty0 491--511, 2000.

\bibitem[Zhang(2002)]{zhang2002assessing}
B.~Zhang.
\newblock Assessing goodness-of-fit of generalized logit models based on
  case-control data.
\newblock \emph{Journal of Multivariate Analysis}, 82\penalty0 (1):\penalty0
  17--38, 2002.

\bibitem[Zhang et~al.(2026)Zhang, Tian, and Li]{zhang2026neyman}
Q.~Zhang, Q.~Tian, and P.~Li.
\newblock Neyman-pearson multiclass classification under label noise via
  empirical likelihood.
\newblock \emph{arXiv preprint arXiv:2603.21623}, 2026.

\bibitem[Zhao et~al.(2018)Zhao, Li, Lai, Suda, Civin, and
  Chandra]{zhao2018federated}
Y.~Zhao, M.~Li, L.~Lai, N.~Suda, D.~Civin, and V.~Chandra.
\newblock Federated learning with {Non-IID} data.
\newblock \emph{arXiv preprint arXiv:1806.00582}, 2018.

\end{thebibliography}
